\documentclass[lettersize,journal]{IEEEtran}
\IEEEoverridecommandlockouts
\usepackage{cite}
\usepackage{amsmath,amssymb,amsfonts}
\usepackage{algorithmic}
\usepackage{graphicx}
\usepackage{textcomp}
\usepackage{subfigure}
\usepackage{xcolor}
\usepackage{xpatch}
\usepackage{amsfonts}
\usepackage{bm}
\usepackage{booktabs}
\usepackage{algorithm}
\usepackage{multirow,bm,bbm,array,setspace}
\usepackage{textcomp}
\usepackage{psfrag}
\usepackage{pstricks,enumerate}
\usepackage{bbm}
\usepackage{threeparttable}
\usepackage{placeins}
\usepackage{stfloats}

\makeatletter
\newcommand{\distas}[1]{\mathbin{\overset{#1}{\kern\z@\sim}}}%
\newsavebox{\mybox}\newsavebox{\mysim}
\newcommand{\distras}[1]{%
  \savebox{\mybox}{\hbox{\kern1pt$\scriptstyle#1$\kern1pt}}%
  \savebox{\mysim}{\hbox{$\sim$}}%
  \mathbin{\overset{#1}{\kern\z@\resizebox{\wd\mybox}{\ht\mysim}{$\sim$}}}%
}
\ExplSyntaxOn
\cs_new:Npn \bibColoredItems #1#2
  {
    \clist_map_inline:nn {#2} {\cs_new:cpn {bib@colored@##1} {#1}} 
}
\ExplSyntaxOff

\newcommand\bib@setcolor[1]{%
  \ifcsname bib@colored@#1\endcsname
    \expanded{\noexpand\color{\csname bib@colored@#1\endcsname}}%
  \else
    \normalcolor
  \fi
}

\IfPackageLoadedTF{hyperref}{\@tempswatrue}{\@tempswafalse}
\if@tempswa
  \xpatchcmd\@bibitem {\H@item}{\bib@setcolor{#1}\H@item}{}{\PatchFailed}
  \xpatchcmd\@lbibitem{\H@item}{\bib@setcolor{#2}\H@item}{}{\PatchFailed}
\else
  \xpatchcmd\@bibitem {\item}  {\bib@setcolor{#1}\item}  {}{\PatchFailed}
  \xpatchcmd\@lbibitem{\item}  {\bib@setcolor{#2}\item}  {}{\PatchFailed}
\fi

\makeatother
\ifCLASSINFOpdf
\else
\fi

\newcommand{\tr}{\mathrm{Tr}}
\newcommand{\bv}{\bm v}
\newcommand{\bE}{\bm E}

\newcommand{\bu}{\bm u}

\newcommand{\bW}{\bm W}
\newcommand{\bx}{\bm x}
\newcommand{\bz}{\bm z}
\newcommand{\bX}{\bm X}
\newcommand{\by}{\bm y}
\newcommand{\bp}{\bm p}
\newcommand{\bY}{\bm Y}

\newcommand{\bD}{\bm D}
\newcommand{\bL}{\bm L}
\newcommand{\bT}{\bm T}

\newcommand{\bI}{\bm I}
\newcommand{\bA}{\bm A}
\newcommand{\ba}{\bm a}
\newcommand{\bB}{\bm B}
\newcommand{\bP}{\bm P}

\newcommand{\bF}{\bm F}

\newcommand{\bZ}{\bm Z}

\newcommand{\bmu}{\bm \mu}

\newtheorem{proposition}{Proposition}

\newtheorem{theorem}{Theorem}

\newtheorem{remark}{Remark}

\begin{document}

\title{Normalized Cuts and Data Clustering via Matrix Fractional Programming}
\author{
\IEEEauthorblockN{
Yannan Chen, \IEEEmembership{Member,~IEEE}, Beichen Huang, Licheng Zhao,\\Chris Ding, and Kaiming Shen, \IEEEmembership{Senior Member,~IEEE}
} 
\thanks{The work of Yannan Chen, Beichen Huang, and Kaiming Shen was supported in part by National Natural Science Foundation of China (NSFC) under Grant 12426306. The work of Licheng Zhao was supported in part by Guangdong Basic and Applied Basic Research Foundation under Grant 2024A1515010154. This paper was presented in part at the 38th Annual Conference on Neural Information Processing Systems (NeurIPS 2024), Vancouver, BC, Canada, December 2024~\cite{chen2024multidimensional}. \emph{(Corresponding author: Kaiming Shen)}.
    
Yannan Chen is with Key Laboratory of Radio Frequency Heterogeneous Integration, Shenzhen University, Shenzhen 518060, China, and School of Science and Engineering, The Chinese University of Hong Kong (Shenzhen), Shenzhen 518172, China (e-mail: yannanchen@link.cuhk.edu.cn).

Beichen Huang is an undergraduate student at McMaster University, Canada (e-mail: huangb21@mcmaster.ca). The work was done during his internship at The Chinese University of Hong Kong (Shenzhen).

Licheng Zhao is with Shenzhen Research Institute of Big Data, Shenzhen 518172, China (e-mail: zhaolicheng@sribd.cn).

Chris Ding is with the School of Data Science, The Chinese University of Hong Kong (Shenzhen), 518172, China (e-mail: chrisding@cuhk.edu.cn).

Kaiming Shen is with School of Science and Engineering, The Chinese University of Hong Kong (Shenzhen), Shenzhen 518172, China (e-mail:  shenkaiming@cuhk.edu.cn).}
}

\maketitle

\begin{abstract}
The normalized cut (NCut) problem is a fundamental one in the area of data clustering. Despite its fractional structure, the NCut problem is seldom considered from a fractional programming (FP) perspective in the existing literature. The main reason is that the conventional FP method is inadequate for this problem case: the classical Dinkelbach's transform can only handle a single ratio and hence is limited to the two-class clustering, whereas the multi-class clustering involves multiple ratio terms. This work proposes a novel so-called multidimensional quadratic transform that is capable of dealing with the multi-ratio case; it recasts the NCut problem into a bipartite matching problem---which can be immediately solved by linear search. Furthermore, we explore the connection between the proposed FP method and the minorization-maximization (MM) theory, and in turn characterize the convergence behavior of our algorithm. Finally, we conduct experiments to numerically demonstrate the advantages of the proposed FP-based algorithm in a variety of data clustering and image segmentation tasks.
\end{abstract}


\section{Introduction}
\IEEEPARstart{F}{ractional} programming (FP) is aimed at those optimization problems involving ratio terms, as extensively considered in the areas of physics, economics, management science, signal processing, computer science, and information theory \cite{stancu2012fractional,bajalinov2003linear,schaible1974parameter}. This paper explores a novel use of FP in solving the normalized cut (NCut) problem for graph clustering \cite{NCut}. Not only the application but also the method of FP is new in this work. We develop the existing quadratic transform \cite{shen2018fractional} into the so-called \emph{multidimensional quadratic transform} to account for multiple ratios in the NCut problem. In contrast, the traditional FP method, Dinkelbach's transform \cite{dinkelbach1967nonlinear}, is limited to the single-ratio case and thus cannot work for the NCut problem. By the proposed new FP method, the original difficult NCut problem with the 0-1 discrete constraint can be converted to a weighted bipartite matching problem that can be efficiently addressed via linear search.


Data clustering has been studied in the literature from a variety of perspectives such as K-means \cite{kmeans}, hierarchical clustering \cite{hierarchical}, spectral clustering (SC) \cite{NCut1}, graph clustering \cite{nie2010improved,zhao2025balanced}, co-clustering\cite{dhillon2003information,xue2024co}, and high-density clustering \cite{ester1996density}. Graph clustering is of particular interest because of its flexibility in data types, not requiring the data clusters to be center-based as those geometry-based clustering algorithms \cite{yan2009fast}. In graph clustering, each data point is visualized as a vertex in a weighted undirected graph, where the weight of each link reflects the correlation between the two incident vertices. The task is to decompose the graph into multiple subgraphs, each subgraph representing a cluster. Intuitively speaking, we wish to put those highly correlated vertices in the same subgraph, while maximizing the ``distance'' between any two subgraphs. Depending on the distance metric, we would obtain different objective functions for the graph clustering problem, e.g., min cut \cite{mincut}, ratio cut \cite{ratiocut}, min-max cut \cite{minmaxcut}, and the NCut \cite{NCut}. This work adopts the NCut metric not only because it is most popular in the area but also because it yields stable performance and prevents cluster imbalance \cite{NCut}.

The NCut problem is a nonlinear optimization problem with the 0-1 discrete constraints, the complexity of which is shown to be NP-complete \cite{NCut}. {The SC method is a conventional approach to the NCut problem \cite{NCut,NCut1}. Its main idea is to recast the NCut problem into the spectral domain and subsequently perform clustering. In this sense, SC and its variants can be viewed as two-stage methods: they first solve a relaxed spectral embedding problem and then recover discrete cluster assignments through a post-processing step such as K-means \cite{NCut}, spectral rotation \cite{huang2013spectral}, or discretization refinement \cite{zhang2023discretize}.} However, the SC method can barely provide any performance guarantees; after all, it is merely a heuristic to turn the NCut problem to the new SC problem. Other approaches resort to the nonconvex optimization theory. For instance, the semidefinite relaxation (SDR) method \cite{bie2006fast} provides a bounded convex approximation for the NCut problem, while \cite{ghanem2010dinkelbach} employs the classical Dinkelbach's algorithm to convexify the NCut problem. However, the aforementioned two methods are confined to the two-class NCut problem.

The multi-class NCut problem is much more challenging. The Fast Coordinate Descent (FCD) algorithm in \cite{FCD} addresses the multi-class NCut problem directly by using the standard method of block coordinate descent. In comparison, the Direct Normalized Cut (DNC) in \cite{DNC} is much less straightforward, which approximates the original problem by using a  bound on the NCut objective. However, the bound proposed in \cite{DNC} incurs a high computational complexity. To address this issue, the Fast Iterative Normalized Cut (FINC) in \cite{FINC} suggests using a different closed-form bound instead. But this new bound in FINC does not lead to an easy problem after approximation; the new problem can only be addressed in a heuristic fashion as discussed in \cite{FINC}. In contrast, the building block of this paper is a novel bound on the NCut objective, which stems from the FP technique and is easy to compute. Most importantly, in light of the proposed lower bound, the original nonlinear discrete NCut problem can be recast into a weighted bipartite matching problem amenable to the linear search method. 
Moreover, it is worth mentioning that the proposed FP approach can be extended to those more complicated {data clustering} tasks such as the multi-view graph clustering \cite{shi2021flexible,qin2022consistency}. {Graph clustering has also been applied in areas such as computer vision\cite{zhong2021graph}, bioinformatics\cite{higham2007spectral}, network analysis\cite{bhowmick2024dgcluster}, and remote sensing \cite{guan2025structure}, as well as mobile-computing scenarios including intelligent transportation\cite{yao2026efl}, UAV coordination\cite{xu2024semantic}, and sensor/IoT data organization \cite{jiang2025representation,li2024demodulation}. These applications suggest that the proposed FP-based NCut optimizer can be used for graph-based inference tasks in broader domains.}

The main results of this paper are summarized in the following:
\begin{itemize}
\item We develop a multidimensional quadratic transform for the FP purpose. This new method is more powerful than the traditional Dinkelbach's method \cite{dinkelbach1967nonlinear} in that it accounts for multiple ratios. The proposed method is also more general than the existing quadratic transform \cite{shen2018fractional,shen2018fractional1} in that it can handle the matrix ratios. This feature plays a key role in converting the NCut problem to a bipartite matching problem.

\item We provide a strengthened convergence analysis of the proposed FP-based clustering method. Building on the minorization-maximization (MM) interpretation established in our preliminary conference work, we further establish stationary-point convergence under differentiability conditions and characterize the local convergence rate of the proposed clustering algorithm.

\item  {We develop a scalable realization of the proposed method for sparse graph settings and discuss how the similarity matrix can be constructed in practice. For small datasets, a fully connected graph constructed from a positive-definite kernel, naturally yields a positive-semidefinite similarity (PSD) matrix. For large datasets, we adopt a sparse anchor-graph construction that also produces a PSD similarity matrix. We further show that an indefinite similarity matrix can be transformed into an equivalent PSD form.}

\item {We evaluate the saturation value of the NCut objective by the FP method. Specifically, we start from the perfect initialization (based on the ground truth) and perform the FP optimization. As a result, the quality measure of clustering (i.e., error) starts at 0\% and increases, even though the objective value is monotonically improved. This experiment reveals a saturation region of the NCut objective. Within this region, solutions with smaller external clustering error cannot be meaningfully distinguished by the NCut objective alone.}
\end{itemize}

An earlier version of this work appeared in NeurIPS 2024\cite{chen2024multidimensional}. Compared with the conference version, the present journal article further strengthens the convergence analysis, develops a scalable sparse anchor-graph realization for large-scale problems, and substantially expands the empirical study with larger datasets, additional spectral-clustering baselines, and the saturation analysis of the NCut objective.

The remainder of this paper is organized as follows. Section \ref{sec:NCut and SC} reviews the NCut problem and the spectral clustering method. Section \ref{sec:FP} reviews the classic FP techniques, and proposes the novel FP technique named multidimensional quadratic transform. We further show the connection between the proposed approach and the MM theory, thereby establishing its convergence behaviors. Section \ref{sec:method} derives the proposed fractional programming-based clustering (FPC) algorithm.  Section
\ref{sec:Experiments} presents numerical results. Finally, Section \ref{sec:conclusion} concludes the
paper.

\emph{Notation:} Here and throughout, bold lower-case letters represent vectors while bold upper-case letters represent matrices. For matrix $\bA$, $\bA^{\top}$ is its transpose, $\bA^{-1}$ is its inverse, $\lambda_{\text{max}}(\bA)$ is its largest eigenvalue, $\bA^{\frac{1}{2}}$ is its symmetric square root\cite{boyd2004convex}, and $\mathrm{vec}(\bA)$ is the vectorization of matrix $\bA$ by stacking its columns to get one column vector. For vector $\ba$, $\ba^{\top}$ is its transpose, and $\|\ba\|$ is its norm. The set of real numbers, the set of nonnegative numbers, and the set of strictly positive numbers are denoted by $\mathbb R$, $\mathbb R_+$, and $\mathbb R_{++}$, respectively. The set of positive semi-definite matrices and the set of positive definite matrices are denoted by $\mathbb{S}_{+}$, and $\mathbb{S}_{++}$, respectively. Moreover, $\bA\preceq \bB$ indicates that $(\bB-\bA)\in\mathbb S_{+}$, and $\bA\prec \bB$ indicates that  $(\bB-\bA)\in\mathbb S_{++}$. 

\section{Preliminary}
\label{sec:NCut and SC}

This section comprises two parts. First, we state the NCut problem for data clustering. Second, we review the classic SC method that serves as a major benchmark in our experiments.

\subsection{NCut Problem for Graph Clustering}
Suppose there are $N$ data points in total. Use $i,j\in\{1,2,\ldots,N\}$ to index these data points. For a pair of data points $i$ and $j$, the similarity between them is quantified as $0\le w_{ij}\le 1$. By symmetry, we have $w_{ij} = w_{ji}$. In the graph theory context, with each data point visualized as a vertex, the edge between vertex $i$ and vertex $j$ is assigned the weight $w_{ij}$ (or $w_{ji}$). Denote by $\mathcal V$ the set of vertices, and $\mathcal E$ the set of edges. The resulting graph $G=(\mathcal V,\mathcal E)$ can be recognized as a weighted undirected graph.

For each vertex $i$, its \emph{degree} $d_i$ is the sum weights across all the incident edges:
\begin{equation}
     d_{i}=\sum_{j=1}^N w_{ij}.
\end{equation}
Dividing the $N$ data points into $K>1$ clusters is equivalent to partitioning $\mathcal V$ into $K$ disjoint subsets $\{\mathcal V_1,\mathcal V_2,\ldots,\mathcal V_K\}$, where $\bigcup^K_{k=1}\mathcal V_k=\mathcal V$ and $\mathcal V_k\cap\mathcal V_{k'}=\emptyset$ for any $k\ne k'$. In principle, data clustering aims to group together those data points that are sufficiently similar to each other. Equivalently, we wish to minimize the similarity between any two clusters. Toward this end, one traditional strategy is to minimize
\begin{equation}
    \mathrm{cut}(\mathcal V_1,\mathcal V_2,\ldots,\mathcal V_K)=\frac{1}{2}\sum_{k=1}^K\sum_{i\in\mathcal{V}_k}\sum_{j\in\bar{\mathcal{V}}_k}w_{ij},
\end{equation}
where $\bar{\mathcal{V}}_k$ is the complement of $\mathcal{V}_k$, i.e., $\bar{\mathcal{V}}_k = \mathcal{V}\backslash\mathcal{V}_k$. An example of two-class clustering is shown in Fig.~\ref{fig:cut}. 
However, minimizing $\mathrm{cut}(\mathcal V_1,\mathcal V_2,\ldots,\mathcal V_K)$ alone can be problematic---it tends to put most data points in one particular cluster while leaving other clusters almost empty, namely the cluster imbalance \cite{von2007tutorial}. 

\begin{figure}[t]
    \centering
    \includegraphics{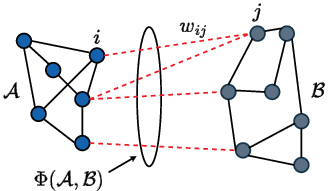}
    \vspace{-0.5em}
    \caption{Graph cut between two disjoint subsets  $\mathcal{A}$ and $\mathcal{B}$.}
    \label{fig:cut}
\end{figure}

To address the above issue, a common method \cite{NCut,NCut1} is to regularize the cluster volume by considering the normalized cut:
\begin{align}
    \mathrm{ncut}(\mathcal V_1,\mathcal V_2,\ldots,\mathcal V_K)&=\frac{1}{2}\sum_{k=1}^K\frac{\sum_{i\in\mathcal{V}_k}\sum_{j\in\bar{\mathcal{V}}_k}w_{ij}}{\sum_{i\in\mathcal{V}_k}d_i}\notag\\
&=\frac{1}{2}\sum_{k=1}^K \frac{\bx_k^\top(\bD - \bW)\bx_k}{\bx_{k}^\top\bD\bx_{k}},
\label{ncut def}
\end{align}
where $\bW=[w_{ij}]\in\mathbb R^{N\times N}$, $\bD=\mathrm{diag}[d_1,d_2,\ldots,d_N]\in\mathbb R^{N\times N}$, and $\bx_k=(x_{1k},\ldots,x_{Nk})^\top\in\{0,1\}^N$, where each entry $x_{ik}\in\{0,1\}$ equals $1$ if data point $i$ is assigned to cluster $k$ and equals $0$ otherwise. We seek the optimal clustering decision $\{x_{ik}\}$ that minimizes $\mathrm{ncut}(\mathcal V_1,\mathcal V_2,\ldots,\mathcal V_K)$. Further, because $\mathrm{ncut}(\mathcal V_1,\mathcal V_2,\ldots,\mathcal V_K)= \frac{1}{2}K-\frac12\sum_{k=1}^K \frac{\bx_k^\top\bW\bx_k}{\bx_{k}^\top\bD\bx_{k}}$,
the NCut minimization problem boils down to
\begin{subequations}
\label{prob:NCut}
\begin{align}
    \underset{\bx_{k}}{\text{maximize}}&\quad \sum_{k=1}^K\frac{\bx_k^\top\bW\bx_k}{\bx_{k}^\top\bD\bx_{k}}
\label{obj:NCut}\\
 \text{subject to}&\quad \sum_{k=1}^K x_{ik} = 1,\; i = 1,\ldots,n \label{cons:unique cluster}\\
        &\quad\, x_{ik}\in\{0,1\}, \; i=1\ldots,n,\;k =1,\ldots. K,\label{cons:discrete value}
\end{align}
\end{subequations}
where the two constraints \eqref{cons:unique cluster} and \eqref{cons:discrete value} state that each data point must be assigned to one unique cluster. The difficulties of the above problem can be recognized with two respects. First, the clustering variables $\{x_{ik}\}$ are discrete. Second, even when every $x_{ik}$ is relaxed to be a continuous variable on the interval $[0,1]$, the problem is still nonconvex.

\begin{remark}
Under the binary constraint $x_{ik}\in\{0,1\}$, we have
\[
\bx_k^\top\bD\bx_k
=
\bm{\delta}^\top\bx_k,
\qquad
\bm{\delta}=[d_1,\ldots,d_N]^\top.
\]
Hence, problem~\eqref{prob:NCut} is equivalently expressed as

\begin{subequations}
\label{prob:NCut-linear}
\begin{align}
    \underset{\bx_{k}}{\text{maximize}}&\quad \sum_{k=1}^{K}
\frac{\bx_k^\top\bW\bx_k}{\bm{\delta}^\top\bx_k},
\label{obj:Linear obj}\\
 \text{subject to}&\quad \eqref{cons:unique cluster}, \eqref{cons:discrete value}.
\end{align}
\end{subequations}
\end{remark}

\subsection{Conventional Spectral Clustering (SC) Method}
\label{subsec:SC}

The SC method was initially proposed 25 years ago \cite{NCut} and remains to be a popular approach to the NCut problem. 
First, the primal problem in \eqref{prob:NCut} can be rewritten as
\begin{subequations}
\label{SC:rewrite prob}
    \begin{align}
         \underset{\bX}{\text{minimize}}&\quad \tr\big((\bX^\top\bD\bX)^{-1}\bX^\top\bL\bX\big)\\
        \text{subject to}&\quad \eqref{cons:unique cluster},\;\eqref{cons:discrete value},
    \end{align}
\end{subequations}
where $\bX=[\bx_1,\ldots,\bx_K]\in\mathbb R^{N\times K}$, and the Laplacian matrix $\bL$ equals $ \bD-\bW$. Letting 
\begin{equation}
\label{P}
\bP = \bX(\bX^\top \bD\bX)^{-\frac12},
\end{equation}
we have $\bP^\top\bD\bP=\bI$. Now we change the variable as $\bT=\bD^{\frac12}\bP$, so that the optimization in $\bX$ is converted to the optimization in $\bT$. The SC method preserves the identity $\bP^\top\bD\bP=\bI$ by requiring $\bT^\top\bT=\bI$, and ignores the original constraints \eqref{cons:unique cluster} and \eqref{cons:discrete value}, thereby approximating \eqref{prob:NCut} as:
\begin{subequations}
\label{prob:NCut SC}
    \begin{align}
        \underset{\bT}{\text{minimize}}&\quad \tr(\bT^\top\bD^{-\frac 12}\bL\bD^{-\frac 12}\bT)\\
        \text{subject to}&\quad \bT^\top\bT=\bI.
    \end{align}
\end{subequations}
We remark that rewriting \eqref{prob:NCut} as \eqref{prob:NCut SC} is completely heuristic. The new problem \eqref{prob:NCut SC} is a standard trace minimization problem that can be efficiently solved as follows. We first find the $K$ largest eigenvalues $\lambda_1\ge\lambda_2\ge\ldots\ge\lambda_K$ of $\bD^{-\frac 12}\bL\bD^{-\frac 12}$ and their associated eigenvectors $\{\bv_1,\bv_2,\ldots,\bv_K\}$ being linearly independent, and then obtain the solution as $\bT=[\bv_1,\bv_2,\ldots,\bv_K]$. Next, we recover the solution of $\bP$ as $\bP=\bD^{-\frac12}\bT$. Finally, we wish to recover the original variable $\bX$ from $\bP$ in \eqref{P}, but this is numerically difficult. A well-known heuristic \cite{von2007tutorial} is to recover $\bX$ by the K-means algorithm \cite{kmeans}. Specifically, letting $\bp_i\in\mathbb R^{1\times K}$ be the $i$th row of $\bP$, we apply the K-means algorithm to group the vectors $\bp_1,\ldots,\bp_N$ into $K$ clusters $\mathcal C_1,\ldots,\mathcal C_K$. Finally, each $x_{ik}$ is determined as 
\begin{equation}
    x_{ik}=
    \begin{cases}
        \;\;1\quad &\text{if}\;\; \bp_i \in \mathcal C_k;\vspace{0.3em}\\
        \;\;0\quad &\text{otherwise.}
    \end{cases}
\end{equation}
To summarize, although the SC method has been extensively considered for the NCut clustering problem, it lacks mathematical justification in two respects: first, the new problem \eqref{prob:NCut SC} may not be equivalent to the original problem \eqref{prob:NCut}; second, after obtaining the solution of $\bP$, the recovery of $\bX$ in \eqref{P} is difficult and can only be addressed heuristically via the K-means algorithm.

\section{Fractional Programming (FP)}
\label{sec:FP}
The NCut problem in \eqref{prob:NCut} is fractionally structured in the sense that its objective function is a sum of ratios. This quick observation strongly suggests that the NCut is amenable to FP, but it turns out that very few previous works in the literature have adopted the FP approach. In the rest of this section, we first review the conventional FP methods to show why they are rarely considered for the NCut, and then introduce a recently proposed FP technique called the multidimensional quadratic transform---which forms the building block of our proposed clustering algorithm as introduced in Section \ref{sec:method}.

\subsection{Existing FP Methods}
\label{subsec:conventional FP methods}

The early studies in the FP field are restricted to the \emph{single-ratio} problem:
\begin{align}
\label{single ratio prob}
    \underset{x\in\mathcal{X}}{\text{maximize}}\quad\frac{A(x)}{B(x)},
\end{align}
where $A(x)\ge0$ is a nonnegative function, $B(x)>0$ is a strictly positive function, and $\mathcal X$ is a nonempty constraint set on $x$. In the literature, many works further assume that $A(x)$ is concave in $x$, $B(x)$ is convex in $x$, and $\mathcal X$ is a convex set, namely the \emph{concave-convex condition}. Notice that problem \eqref{single ratio prob} is still nonconvex in general under the concave-convex condition, so the direct solving of \eqref{single ratio prob} is difficult. The classic Dinkelbach's transform in essence aims to decouple the ratio:

\begin{proposition}[Dinkelbach's Transform \cite{dinkelbach1967nonlinear}]
\label{prop:Dinkelbach}
The single-ratio problem \eqref{single ratio prob} is equivalent to
\begin{align}
\label{prob Dinkelbach}
    \underset{x\in\mathcal{X}}{\text{maximize}}\quad A(x)-{yB(x)},
\end{align}
where the auxiliary variable $y$ is iteratively updated as $y={A(x)}/{B(x)}$.
\end{proposition}

Observe that the new problem \eqref{prob Dinkelbach} is convex in $x$ for fixed $y$ under the concave-convex condition, and hence can be efficiently solved by the standard optimization method. Importantly, solving for $x$ in \eqref{prob Dinkelbach} with $y$ iteratively updated guarantees convergence to the global optimum of the original problem \eqref{single ratio prob}. However, it is difficult to extend Dinkelbach's transform to the multi-ratio problems (except for the max-min-ratios case \cite{crouzeix1985algorithm}). As such, the use of Dinkelbach's transform in the NCut area is limited to the two-class clustering that only needs to optimize a single ratio \cite{ghanem2010dinkelbach}. 

We now consider $K>1$ pairs of the numerator function $A_k(x)\ge0$ and denominator function $B_k(x)>0$ along with a nonempty constraint set $\mathcal X$. A \emph{sum-of-ratios} problem is then formulated as
\begin{equation}
\label{prob:scalar FP}
    \underset{x\in\mathcal{X}}{\text{maximize}}\quad \sum_{k=1}^K\frac{A_k(x)}{B_k(x)}.
\end{equation}
It is tempting to decouple each ratio $A_k(x)/B_k(x)$ by using Dinkelbach's transform separately, but the resulting new problem is not equivalent to problem \eqref{prob:scalar FP}. Consequently, the classic Dinkelbach's transform does not work for the NCut with general $K$ clusters. A valid method to decouple multiple ratios is presented in the following proposition.

\begin{proposition}[Quadratic Transform \cite{shen2018fractional}]
\label{prop:scalar QT}
    The sum-of-ratios problem \eqref{prob:scalar FP} is equivalent to 
    \begin{align}  
     \label{prob:scalar FP QT}
    \underset{x\in\mathcal{X},\, y_k\in\mathbb{R}}{\text{maximize}}\quad &\sum_{k=1}^K \Big(2y_k\sqrt{A_k(x)}-y_k^2B_k(x)\Big),
    \end{align}
in the sense that $x^\star$ is a solution to \eqref{prob:scalar FP} if and only if $(x^\star,y^\star)$ is a solution to \eqref{prob:scalar FP QT}, where an auxiliary variable $y_k$ is introduced for each ratio term $A_k(x)/B_k(x)$. 
\end{proposition}

We propose optimizing $x$ and $\{y_k\}$ iteratively. When $x$ is held fixed, each $y_k$ can be optimally determined as
\begin{align}
    y^\star_k=\frac{\sqrt{A_k(x)}}{B_k(x)}.
\end{align}
Furthermore, under a \emph{generalized concave-convex condition} \cite{shen2018fractional} wherein each $A_k(x)$ is a concave function, each $B_k(x)$ is a convex function, and $\mathcal X$ is a convex set, it can be shown that the new problem \eqref{prob:scalar FP QT} is convex in $x$ when $\{y_k\}$ are held fixed. Thus, the alternating optimization between $x$ and $\{y_k\}$ can be performed efficiently.

Now let us return to the NCut problem in \eqref{prob:NCut} and apply the above FP technique to it. Treating $\bx_k^\top\bW\bx_k$ and $\bx_{k}^\top\bD\bx_{k}$ respectively as numerator and denominator, we can recast problem \eqref{prob:NCut} into
\begin{subequations}
\label{prob:NCut scalar QT}
\begin{align}
\underset{\bx_{k},\,y_k\in\mathbb R}{\text{maximize}}&\quad \sum_{k=1}^K\left(2y_k\sqrt{\bx_k^\top\bW\bx_k}-y^2_k{\bx_{k}^\top\bD\bx_{k}}\right)
\label{obj:NCut scalar QT}\\
 \text{subject to}&\quad \eqref{cons:unique cluster},\;\eqref{cons:discrete value}.
\end{align}
\end{subequations}
As before, we optimize $\{\bx_{k}\}$ and $\{y_k\}$ iteratively. For fixed $\{\bx_{k}\}$, the optimal solution of $y_k$ is
\begin{equation}
    y_k^\star = \frac{\sqrt{\bx_k^\top\bW\bx_k}}{\bx_{k}^\top\bD\bx_{k}}.
\end{equation}
It remains to optimize $\{\bx_{k}\}$ in \eqref{prob:NCut scalar QT} for fixed $\{y_k\}$. Observe that $\bx^\top\bW\bx$ is not a concave function of $\bx$ since $\bW$ is often a positive semi-definite matrix \cite{scholkopf1999advances}, and also that the constraint set of $\{\bx_{k}\}$ is not convex because of \eqref{cons:discrete value}, so the aforementioned generalized concave-convex condition does not hold here. As a result, solving for $\{\bx_{k}\}$ in \eqref{prob:NCut scalar QT} with $\{y_k\}$ held fixed is no longer a convex problem. The above alternating optimization between $\{\bx_{k}\}$ and $\{y_k\}$ can be recognized as the so-called \emph{Fast Iterative Normalized Cut (FINC)} algorithm of the recent work \cite{FINC}. Although it manages to decouple multiple ratios in the NCut problem, we are faced with a new challenging problem. The new problem is dealt with in a heuristic fashion in \cite{FINC}.
This fact perhaps explains why the FP approach has not yet been considered extensively in the literature despite the fractional structure of the NCut problem.

\subsection{Proposed Multidimensional Quadratic Transform}

We now proceed to a much more sophisticated FP toolkit that accounts for multidimensional ratios. To start, consider the following matrix extension of the traditional scalar-valued FP problem: each $A_k(x)\ge0$ is generalized as positive semi-definite $\bA_k(x)\in\mathbb S^{m\times m}_{+}$, while each $B_k(x)>0$ is generalized as positive definite $\bB_k(x)\in\mathbb S^{m\times m}_{++}$. Accordingly, the ratio term is extended to the matrix form as
\begin{equation*}
    \frac{A_k(x)}{B_k(x)} \;\Longrightarrow\; \bB_k(x)^{-1}\bA_k(x).
\end{equation*}
We then arrive at a matrix extension of the sum-of-ratios problem \eqref{prob:scalar FP}:
\begin{align}
\label{prob:matrix FP}
    \underset{x\in\mathcal{X}}{\text{maximize}}\quad \sum_{k=1}^K\mathrm{Tr}\left(\bB_k^{-1}(x)\bA_k(x)\right).
\end{align}
One main result of this paper is that the quadratic transform in Proposition \ref{prop:scalar QT} carries over to the matrix ratio case, as stated in the following proposition.

\begin{theorem}[Multidimensional Quadratic Transform]
\label{prop matrix QT}
Suppose that each $\bA_k(x)\in\mathbb S^{m\times m}_{+}$ can be factorized as
\begin{equation}
    \bA_k(x) = [\bZ_k(x)]^\top[\bZ_k(x)]\quad\text{where}\;\; \bZ_k(x)\in\mathbb R^{\ell\times m}
\end{equation}
for some positive integer $\ell$. The matrix FP problem \eqref{prob:matrix FP} is then equivalent to
\begin{align}
\label{prob:matrix FP QT}
\underset{x\in\mathcal{X},\,\bY_{k}\in\mathbb{R}^{\ell\times m}}{\text{maximize}}\;&\sum_{k=1}^K\mathrm{Tr}\left(2\bY_k[\bZ_k(x)]^\top-\bY_k \bB_k(x)\bY_k^\top\right),
\end{align}
\end{theorem}
where an auxiliary variable $\bY_k\in\mathbb R^{\ell\times m}$ is introduced for each matrix ratio $\bB^{-1}_k(x)\bA_k(x)$.

\begin{IEEEproof} It can be shown that each $\bY_k$ in \eqref{prob:matrix FP QT} is always optimally determined as
\begin{align}
    \label{eq:update matrix Y}
    \bY^\star_k=\bZ_k(x)\bB_k^{-1}(x).
\end{align}
Substituting the above $\bY^\star_k$ in \eqref{prob:matrix FP QT} recovers the original problem \eqref{prob:matrix FP}.
\end{IEEEproof}
It is easy to see that the multidimensional quadratic transform in Theorem \ref{prop matrix QT} encompasses the conventional quadratic transform in Proposition \ref{prop:scalar QT} as a special case.

{\begin{remark}
The factorization $\bA_k(x)=\bZ_k(x)^\top\bZ_k(x)$ in Theorem~\ref{prop matrix QT} is not an additional modeling assumption; it follows from the positive semidefiniteness of $\bA_k(x)$, for example by taking a valid matrix square root. The role of Theorem~\ref{prop matrix QT} is therefore to provide an exact quadratic transform of the original matrix fractional objective. This differs from symmetric nonnegative matrix factorization\cite{kuang2012symmetric}, which treats graph clustering through a low-rank nonnegative factorization of a similarity matrix, and can be viewed as operating on a relaxed factorized formulation. In contrast, Theorem~\ref{prop matrix QT} is applied to the original NCut-related matrix fractional objective and preserves its optimizer.
\end{remark}}

{Theorem~\ref{prop matrix QT} naturally suggests an alternating optimization procedure for solving \eqref{prob:matrix FP}. Specifically, for any fixed $x$, each auxiliary variable $\bY_k$ admits the closed-form update in \eqref{eq:update matrix Y}. On the other hand, when $\{\bY_k\}$ are fixed, the reformulated objective in \eqref{prob:matrix FP QT} becomes an optimization problem only with respect to $x$. Therefore, instead of directly handling the original matrix fractional problem \eqref{prob:matrix FP}, we can iteratively update $\{\bY_k\}$ and $x$ based on the transformed formulation. This generic procedure is summarized in Algorithm~\ref{algorithm:general matrix FP}.}
\begin{algorithm}[t]
  \caption{Iterative Method for Solving Problem \eqref{prob:matrix FP}}
  \label{algorithm:general matrix FP}
  \begin{algorithmic}[1]
      \STATE Initialize $x$ to a feasible value.\\
      \REPEAT 
      \STATE Update each $\bY_k$ according to \eqref{eq:update matrix Y}, for $k=1,2,\ldots,K$.
      \STATE Update $x$ by solving the reformulated problem \eqref{prob:matrix FP QT}.
      \UNTIL{the original objective value converges} 
  \end{algorithmic}
\end{algorithm}

\subsection{Connection to MM Theory}
It turns out that our proposed multidimensional quadratic transform has a close connection with the minorization-maximization method. Based on this connection, we further show the convergence behavior of Algorithm \ref{algorithm:FP}. 
We begin with a brief review of the MM theory \cite{sun2016majorization,razaviyayn2013unified}. Consider a general constrained optimization problem 
\begin{align}
\label{MM original prob}
    \underset{x\in\mathcal{X}}{\text{maximize}}\quad f(x).
\end{align}
Rather than solving the above problem directly, the MM theory deals with an approximation of \eqref{MM original prob} iteratively. In principle, 
we construct a \emph{surrogate function} $g(x|\hat{x})$ of $x$ given the condition parameter $\hat x\in\mathcal{X}$, such that
\begin{align}
   \label{g le}
    g(x|\hat x)&\le f(x), \\
    g(\hat x|\hat x)&= f(\hat x).
    \label{g eq}
\end{align}
Then the MM method solves a sequence of new problems of the surrogate function:
\begin{align}
\label{MM prob}
    \underset{x\in\mathcal{X}}{\text{maximize}}\quad g(x|\hat x),
\end{align}
with $\hat x$ iteratively updated to the previous solution $x$. Specifically, with the solution of \eqref{MM prob} in the previous iteration denoted by $\hat x$, we construct a surrogate function $g(x|\hat x)$ for the current iteration and then obtain the new solution $\hat x$ of \eqref{MM prob}, and so forth, as illustrated in Fig.~\ref{fig:MM}.

\vspace{1em}

\begin{figure}[t]
  \centering
  \includegraphics[width=1\linewidth]{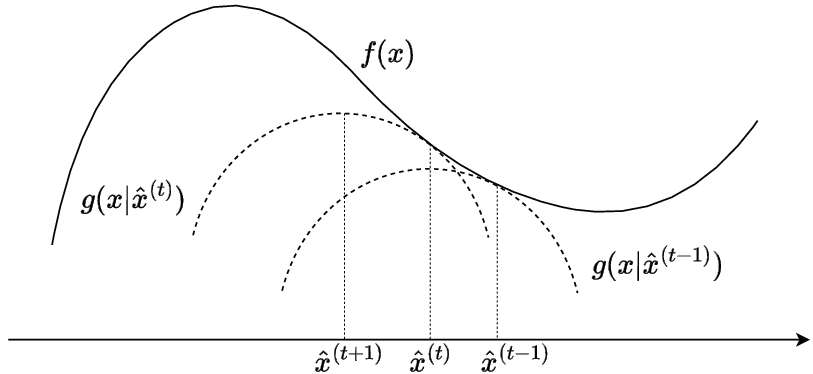}
  \vspace{1em}
  \caption{The value of the original objective $f(x)$ is nondecreasing after each iteration of the MM method, where $g(x|\hat x^{(t)})$ is the surrogate function constructed in the $t$th iteration.}
  \label{fig:MM}
\end{figure}

We now show that the alternating optimization between $x$ and $\{\bY_k\}$ in \eqref{prob:matrix FP QT} can be interpreted as an MM method.  
We now view the update of $\bY_k$ as a function of the previous solution $\hat x$, that is
\begin{align}
    \mathcal Y_{k}(\hat x)=\bZ_k(\hat x)\bB^{-1}_k(\hat x).
\end{align}
Substituting each $\bY_k$ with $\mathcal Y_k(\hat x)$ in the new objective function in \eqref{prob:matrix FP QT} gives rise to a function of $x$ conditioned on $\hat x$:
\begin{multline}
\label{g x|x}
    g(x|\hat x)=\sum_{k=1}^K \mathrm{Tr}(2\bZ_k(\hat x)\bB^{-1}_k(\hat x)[\bZ_k(x)]^\top\\-\bZ_k(\hat x)\bB^{-1}_k(\hat x)\bB_k(x)[\bZ_k(\hat x)\bB^{-1}_k(\hat x)]^\top)
    ).
\end{multline}
Notice that maximizing the objective function in \eqref{prob:matrix FP QT} with $\{\bY_k\}$ fixed is equivalent to maximizing $g(x|\hat x)$ where $\hat x$ is fixed at the solution of the previous iteration.

Most importantly, it can be shown that  $g(x|\hat x)$ meets the conditions \eqref{g le} and \eqref{g eq} for the original objective function in \eqref{prob:matrix FP}, so the alternating optimization between $x$ and $\{\bY_k\}$ in \eqref{prob:matrix FP QT} amounts to an MM procedure. In light of the MM theory, we immediately have the following result. 
\begin{proposition}
\label{prop:nondecreasing}
The alternating optimization between $x$ and $\{\bY_k\}$ in \eqref{prob:matrix FP QT} guarantees a nondecreasing convergence of the original optimization objective in \eqref{prob:matrix FP}. Furthermore, if functions $\bZ_k(x)$ and $\bB_k(x)$ are differentiable, then the alternating optimization converges to a stationary point of the original problem \eqref{prob:matrix FP}.
\end{proposition}

We now return to the NCut problem in \eqref{prob:NCut}. Similar to quadratic transform in Proposition \ref{prop:scalar QT}, the optimization of the auxiliary variable $\{\bY_k$\} with fixed $x$ can be found in closed form as given in \eqref{eq:update matrix Y}. The key step is to optimize $x$ for fixed $\{\bY_k\}$. Recall that the primal variable $x$ is still difficult to optimize for the NCut problem after applying the quadratic transform. In contrast, it turns out that the multidimensional quadratic transform in Proposition \ref{prop matrix QT} can lead us to an efficient iterative update of $x$ for the NCut problem scenario, as elaborated in the next section.

\section{FP-based approach to {NCut} problem}
\label{sec:method}

The goal of this section is to address problem \eqref{prob:NCut} by means of the multidimensional FP. We begin with a special case in which the similarity matrix $\bW$ is assumed to be positive semi-definite. Next, we discuss the indefinite $\bW$ case and how the similarity matrix is generated in the different task scenarios.

\subsection{Proposed Multidimensional FP-Based Clustering}
It is crucial to note that the numerator part $\bx_k^\top\bW\bx_k$ can be factorized as
\begin{equation}
    \bx_k^\top\bW\bx_k = \bz_k^\top\bz_k\quad\text{where}\;\; \bz_k = \bW^{\frac12}\bx_k.
\end{equation}
We remark that $\bW^{\frac12}$, i.e., the symmetric square root of $\bW$ \cite{boyd2004convex}, is guaranteed to exist because we have assumed that $\bW\in\mathbb S^{m\times m}_{+}$ for the current discussion. We then treat $\bx_k^\top\bW\bx_k$, $\bz_k$, and $\bx_{k}^\top\bD\bx_{k}$ as $\bA_k(x)$, $\bZ_k(x)$, and $\bB_k(x)$, respectively, in Theorem \ref{prop matrix QT}, with $m=1$ and $\ell=N$,
thus using the multidimensional quadratic transform to reformulate the NCut problem \eqref{prob:NCut} as
\begin{subequations}
\label{prob:NCut matrix FP}
\begin{align}
\underset{\bx_{k},\,\by_{k}\in\mathbb{R}^{N}}{\text{maximize}}&\quad \sum_{k=1}^K \mathrm{Tr}\left(2\by_{k}( {\bW}^{\frac{1}{2}}\bx_k)^\top-\by_{k}(\bx_{k}^\top\bD\bx_{k})\by_{k}^\top\right)\label{prob:NCut matrix FP:obj}\\
\text{subject to}&\quad \eqref{cons:unique cluster},\;\eqref{cons:discrete value}.
\end{align}
\end{subequations}
We then optimize $\bx_{k}$ and $\{\by_k\}$ iteratively. When $\{\bx_{k}\}$ is held fixed, each $\by_k$ is optimally determined as
\begin{equation}
\label{opt y}
\by^\star_k=\frac{{\bW}^{\frac{1}{2}}\bx_{k}}{\bx_k^\top\bD\bx_k}.
\end{equation}
Now the core question is whether $\{\bx_{k}\}$ could be efficiently solved when $\{\by_k\}$ are held fixed. The optimization objective of $\{\bx_{k}\}$ in \eqref{prob:NCut matrix FP:obj} for fixed $\{\by_k\}$ is written as
\begin{equation}
\label{h(X)}
   h(\bx_1,\ldots,\bx_{K}) = \sum_{k=1}^K \mathrm{Tr}\left(2\by_{k}( {\bW}^{\frac{1}{2}}\bx_k)^\top-\by_{k}(\bx_{k}^\top\bD\bx_{k})\by_{k}^\top\right).
\end{equation}
Recall from the equivalent formulation in problem~\eqref{prob:NCut-linear}
that, under the binary constraint $x_{ik}\in\{0,1\}$,
\begin{equation}
\label{xDx=Dx}
\bx_k^\top\bD\bx_k=\bm\delta^\top\bx_k.
\end{equation}
We can then rewrite $h(\bx_1,\ldots,\bx_K)$ as
\begin{align}
   h(\bx_1,\ldots,\bx_{K}) &= \sum_{k=1}^K \left(2\by_k^\top{\bW}^{\frac{1}{2}}\bx_{k}-\by_{k}^\top\by_{k}\bm\delta^\top\bx_{k}\right)\notag\\
   &= \sum^K_{k=1}\bmu^\top_k\bx_k,\label{f chain}
\end{align}
where 
\begin{equation}
\label{h}
    \bmu_k = 2\bW^{\frac12}\by_k-\bm\delta\by^\top_k\by_k.
\end{equation}
In \eqref{f chain}, the first equality follows since $\bx^\top_k\bD\bx_k$ is a scalar. Denote by $\mu_{ik}$ the $i$th component of $\bmu_k$. In light of \eqref{f chain}, we can readily maximize $h(\{\bx_{k}\})$ under the constraints \eqref{cons:unique cluster} and \eqref{cons:discrete value}. Intuitively, for each data point $i$, it is optimal to set $x_{ik}$ with the largest $\mu_{ik}$ to one, while setting the rest $x_{ik}$ to zero, i.e.,
\begin{equation}
\label{eq:update X}
    x^\star_{ik}=
    \begin{cases}
        \;\;1\quad &\text{if}\;\; k = \mathop{\arg}\limits_{k'}\max\, \mu_{ik'};\vspace{0.3em}\\
        \;\;0\quad &\text{otherwise.}
    \end{cases}
\end{equation}
If there exists a tie (i.e., when more than one cluster index $k'$ maximizes $\mu_{ik'}$ for same $i$) then break it randomly. Further, with $\by^\star_k$ in \eqref{opt y} plugged in \eqref{h}, we obtain an efficient computation of $\bmu_k$ as
\begin{align}
\bmu_k &= \frac{2{\bW}\bx_{k}}{\bx_k^\top\bD\bx_k}-\frac{\bm\delta\bx_k^\top\bW\bx_{k}}{(\bx_k^\top\bD\bx_k)^2}\notag\\
&=\frac{2{\bW}\bx_{k}}{\bm\delta^\top\bx_k}-\frac{\bm\delta\bx_k^\top\bW\bx_{k}}{(\bm\delta^\top\bx_k)^2}.
\label{eq:update H}
\end{align}
The merits of rewriting $\bmu_k$ as \eqref{eq:update H} are two-fold. First, it sidesteps the update of the auxiliary variables $\{\by_k\}$. Second, it no longer entails computing the square root of $\bW$. The resulting algorithm referred to as \emph{fractional programming-based clustering (FPC)} is summarized in the following.

\begin{algorithm}[t]
  \caption{Fractional Programming-Based Clustering}
  \label{algorithm:FP}
  \begin{algorithmic}[1]
      \STATE Initialize $\{\bx_{k}\}$ to a feasible value satisfying $\sum^K_{k=1}x_{ik}=1$ and $x_{ik}\in\{0,1\}$ for each $(i,k)$.\\
      \REPEAT 
      \STATE Update each $\bmu_k$ according to \eqref{eq:update H}.
      \STATE Update each $x_{ik}$ according to \eqref{eq:update X}.
      \UNTIL{the value of $h(\{\bx_{k}\})=\sum^K_{k=1}\bmu^\top_k\bx_k$ converges.} 
  \end{algorithmic}
\end{algorithm}

Clearly, the FPC algorithm is guaranteed to converge in terms of the new objective value $h(\bx_1,\ldots,\bx_{K})$, since the iterative update of $\{\bx_{k}\}$ and $\{\bmu_k\}$ in FPC amounts to a block coordinate ascent for problem \eqref{prob:NCut matrix FP} so that $h(\bx_1,\ldots,\bx_{K})$ is monotonically increasing after each iteration. We can actually claim a stronger result for FPC according to Proposition \ref{prop:nondecreasing}, as stated in the subsequent proposition.

\begin{proposition}
\label{prop:convergence}
    Not only the new objective value $h(\bx_1,\ldots,\bx_{K})$ in \eqref{h(X)} but also the original objective value of the NCut problem, $f(\bx_1,\ldots,\bx_{K})=\sum_{k=1}^K\frac{\bx_k^\top\bW\bx_k}{\bx_{k}^\top\bD\bx_{k}}$, is nondecreasing after each iteration of FPC.
\end{proposition}

{ Since the general NCut problem with discrete constraints is NP-complete \cite{NCut}, this monotonic convergence result should not be interpreted as a global-optimality guarantee; our analysis focuses on monotonic improvement of the NCut objective and local convergence behavior of the iterative procedure.}

\subsection{Convergence Rate Analysis}
\label{subsec:rate}

Moreover, in light of the MM theory \cite{shen2024accelerating},
we can characterize the convergence rate of Algorithm
\ref{algorithm:FP}. To facilitate the analysis, we see the
objective function of the NCut problem as a function of
$\bar\bx$, where
$\bar\bx=[\bx_1,\bx_2,\ldots,\bx_K]$ is an $NK\times 1$
vector.
Denote by $f(\cdot)$ the NCut objective value,
$\bar\bx^*$ the local maximizer under consideration,
$t$ the iterate index, $R$ the radius of the local neighborhood
around $\bar\bx^*$, $L$ the Lipschitz constant of
$\nabla^2 f(\bar\bx)$, $h(\bar\bx,\bar\by)$ the new objective
function defined below, where
$\bar\by=[\by_1,\by_2,\ldots,\by_K]$, and $\bar\bx^{t}$ the
solution after $t$ iterates.
Following the analysis steps in \cite{shen2024accelerating},
we characterize the convergence rate of Algorithm
\ref{algorithm:FP} as stated in the following proposition.

\begin{proposition}
\label{theorem:converge rate}
If the Hessian of the objective function $f(\bar\bx)$ of the
continuous relaxation of problem~\eqref{prob:NCut-linear} is
$L$-Lipschitz continuous, then the local
convergence rate of the alternating optimization is
\begin{align*}
|f(\bar\bx^*)-f(\bar\bx^{1})|
&\le
\frac{\Lambda R^2}{2}+\frac{LR^3}{6},\\
|f(\bar\bx^*)-f(\bar\bx^{t})|
&\le
\frac{2\Lambda R^2+2LR^3/3}{t+3},
\quad\text{for}\; t\ge 2.
\end{align*}
\end{proposition}

The rest of Section \ref{subsec:rate} is devoted to the proof
of the above proposition. Recall that problems~\eqref{prob:NCut}
and~\eqref{prob:NCut-linear} are equivalent under the binary
constraint, and therefore lead to the same FPC updates. For the
convergence-rate analysis, however, we relax the binary constraint
of problem~\eqref{prob:NCut-linear} as $0\le x_{ik}\le1$.

Accordingly, in this subsection we denote
\begin{equation}
\label{rate:f}
f(\bar\bx)
=
\sum_{k=1}^{K}
\frac{\bx_k^\top\bW\bx_k}
{\bm\delta^\top\bx_k},
\end{equation}
and use its quadratic-transform surrogate
\begin{equation}
\label{rate:h}
h(\bar\bx,\bar\by)
=
\sum_{k=1}^{K}
\left(
2\by_k^\top\bW^{\frac12}\bx_k
-
\by_k^\top\by_k\,\bm\delta^\top\bx_k
\right).
\end{equation}
For fixed $\bar\by$, $h(\bar\bx,\bar\by)$ is linear in
$\bar\bx$.

Note that the original NCut problem is nonconvex, and the global analysis is intractable. We would like to give a local analysis by restricting the constraint set to a small neighborhood
of a local maximizer, i.e.,
\begin{align}
\label{small_constraint}
\mathcal{X}
=
\{\bar\bx:\|\bar\bx-\bar\bx^*\|_2\le R\},
\end{align}
where $\bar\bx^*$ is the local maximizer under consideration
and the radius $R>0$ is sufficiently small. In the following
local analysis, we assume that $f(\bar\bx)$ is concave on
$\mathcal{X}$ and focus on the iterates that remain in
$\mathcal{X}$.

The convergence rate analysis also uses the MM interpretation.
Conditioned on $\bar\bx'\in\mathcal X$, write the difference
between $f(\bar\bx)$ and the surrogate function
$h(\bar\bx,\bar\by)$ as a function of
$\bar\bx\in\mathcal X$:
\begin{align*}
&\rho(\bar\bx|\bar\bx')\\
&=f(\bar\bx)-h(\bar\bx,\mathcal{Y}(\bar\bx'))\\
&=
\sum_{k=1}^K\!
\frac{\bx_k^\top\bW\bx_k}{\bm\delta^\top\bx_k}
-\!\sum_{k=1}^K\!\big(
2\by_k^\top\!(\bx_k')\bW^{\frac12}\bx_k
-\|\by_k(\bx_k')\|_2^2\bm\delta^\top\bx_k
\big).
\end{align*}
It can be readily shown that
\begin{subequations}
\label{delta:conditions}
\begin{align}
\rho(\bar\bx^t|\bar\bx^t)
&=0,\\
\nabla\rho(\bar\bx^t|\bar\bx^t)
&=\bm{0}.
\end{align}
\end{subequations}

Moreover, define the following quantity as
\begin{align}
\label{rate:Lambda}
\Lambda
=
\sup_{\bar\bx\in\mathcal X}
\lambda_{\max}
\left(
\nabla^{2}
\rho(\bar\bx|\bar\bx)
\right).
\end{align}
Since $h(\bar\bx,\mathcal{Y}(\bar\bx'))$ is linear in
$\bar\bx$ for fixed $\bar\bx'$, the Hessian of
$\rho(\bar\bx|\bar\bx')$ with respect to $\bar\bx$ has the
same Lipschitz constant $L$ as $\nabla^2 f(\bar\bx)$. Lemma 1.2.4 in \cite{Nesterov_book} states that for any
twice-differentiable function $\phi(x)$ with an $L$-Lipschitz
continuous Hessian, we have
\begin{multline}
\bigg|
\phi(x')-\phi(x)-\nabla \phi(x)^\top(x'-x)\\
-\frac{1}{2}(x'-x)^\top
\nabla^2\phi(x)(x'-x)
\bigg|
\le
\frac{L}{6}\|x'-x\|_2^3
\end{multline}
given any two feasible $x$ and $x'$.

Applying the above lemma to the function
$\rho(\bar\bx|\bar\bx^{t-1})$ as a function of $\bar\bx$,
with $\bar\bx^{t-1}$ fixed, and using the results in
\eqref{delta:conditions} give
\begin{align}
&\frac{L}{6}\|\bar\bx-\bar\bx^{t-1}\|_2^3
\notag\\
&\ge
\rho(\bar\bx|\bar\bx^{t-1})
\notag\\
&\quad
-\frac{1}{2}
(\bar\bx-\bar\bx^{t-1})^\top
\nabla^2
\rho(\bar\bx^{t-1}|\bar\bx^{t-1})
(\bar\bx-\bar\bx^{t-1})
\notag\\
&\ge
\rho(\bar\bx|\bar\bx^{t-1})
-\frac{\Lambda}{2}
\|\bar\bx-\bar\bx^{t-1}\|_2^2
\notag\\
&=
f(\bar\bx)
-h(\bar\bx,\bar\by^t)
-\frac{\Lambda}{2}
\|\bar\bx-\bar\bx^{t-1}\|_2^2
\notag\\
&\overset{(a)}{\ge}
f(\bar\bx)
-h(\bar\bx^t,\bar\by^t)
-\frac{\Lambda}{2}
\|\bar\bx-\bar\bx^{t-1}\|_2^2
\notag\\
&\overset{(b)}{\ge}
f(\bar\bx)
-h(\bar\bx^t,\bar\by^{t+1})
-\frac{\Lambda}{2}
\|\bar\bx-\bar\bx^{t-1}\|_2^2
\notag\\
&\overset{(c)}{=}
f(\bar\bx)
-f(\bar\bx^{t})
-\frac{\Lambda}{2}
\|\bar\bx-\bar\bx^{t-1}\|_2^2,
\label{rate:L}
\end{align}
where step $(a)$ follows since $\bar\bx^t$ maximizes
$h(\bar\bx,\bar\by)$ for the current
$\bar\by=\bar\by^t$, step $(b)$ follows since
$\bar\by^{t+1}$ maximizes $h(\bar\bx,\bar\by)$ for the current
$\bar\bx=\bar\bx^t$, and step $(c)$ follows by the property of
the surrogate function.

Following Nesterov's proof technique in \cite{Nesterov_book},
we let
\begin{equation}
\label{rate:x}
\bar\bx
=
\pi\bar\bx^*
+
(1-\pi)\bar\bx^{t-1},
\end{equation}
where $\pi\in[0,1]$. Then the concavity of $f(\bar\bx)$ on
$\mathcal X$ gives
\begin{equation}
\label{rate:concave}
f(\bar\bx)
\ge
\pi f(\bar\bx^*)
+
(1-\pi)f(\bar\bx^{t-1}).
\end{equation}

Denote the gap in the objective value as
\begin{equation}
v_t
=
f(\bar\bx^*)-f(\bar\bx^{t}).
\end{equation}
Substituting \eqref{rate:x} and \eqref{rate:concave} into
\eqref{rate:L} gives rise to
\begin{align}
v_t
&\le
(1-\pi)v_{t-1}
+
\frac{\pi^2\Lambda}{2}
\|\bar\bx^*-\bar\bx^{t-1}\|_2^2
\notag\\
&\qquad
+
\frac{\pi^3L}{6}
\|\bar\bx^*-\bar\bx^{t-1}\|_2^3
\notag\\
&\le
(1-\pi)v_{t-1}
+
\pi^2
\bigg(
\frac{\Lambda R^2}{2}
+
\frac{LR^3}{6}
\bigg),
\label{rate:v}
\end{align}
where the second inequality follows by
\eqref{small_constraint} and $0\le\pi\le1$.
The choice of $\pi$ depends on $t$.

When $t=1$, we let $\pi=1$ in \eqref{rate:v} and obtain
\begin{equation}
\label{rate:v1}
v_1
\le
\frac{\Lambda R^2}{2}
+
\frac{LR^3}{6}.
\end{equation}
When $t\ge2$, we let
\begin{equation}
\label{rate:pi}
\pi
=
\frac{v_{t-1}}
{\Lambda R^2+LR^3/3}.
\end{equation}
It can be shown by induction that the above $\pi$ is always
feasible (i.e., $0\le\pi\le1$) for all $t\ge2$.
Plugging \eqref{rate:pi} in \eqref{rate:v} yields
\begin{equation}
v_t
\le
v_{t-1}
\bigg(
1-
\frac{v_{t-1}}
{2\Lambda R^2+2LR^3/3}
\bigg),
\end{equation}
which can be further rewritten as follows. If $v_t=0$, the
desired bound holds immediately. Otherwise,
\begin{align}
\frac{1}{v_t}
&\ge
\frac{1}{v_{t-1}}
\cdot
\bigg(
1-
\frac{v_{t-1}}
{2\Lambda R^2+2LR^3/3}
\bigg)^{-1}
\notag\\
&\ge
\frac{1}{v_{t-1}}
\cdot
\bigg(
1+
\frac{v_{t-1}}
{2\Lambda R^2+2LR^3/3}
\bigg)
\notag\\
&=
\frac{1}{v_{t-1}}
+
\frac{1}
{2\Lambda R^2+2LR^3/3},
\label{rate:1/v}
\end{align}
where the second inequality follows since
$(1-a)^{-1}\ge1+a$ for any $0\le a<1$.
The result of \eqref{rate:1/v} immediately gives
\begin{align}
\frac{1}{v_t}
&\ge
\frac{1}{v_1}
+
\frac{t-1}
{2\Lambda R^2+2LR^3/3}
\notag\\
&\ge
\frac{t+3}
{2\Lambda R^2+2LR^3/3},
\end{align}
where the second inequality is due to \eqref{rate:v1}.
The proof is then completed.

{\subsection{Similarity Matrix Generation}
\label{subsec: similarity matrix}

We now discuss how the similarity matrix $\bW$ in problem \eqref{prob:NCut} is generated given a particular dataset. If the dataset is sufficiently small so that the graph representation is fully connected, i.e., each vertex is linked to the rest vertices, then it is common to use a positive definite kernel \cite{fasshauer2011positive} (e.g., the Gaussian kernel) to generate the similarity matrix. The similarity between any two data points based on the Gaussian kernel is given by
    \begin{align*}
        w_{ij} = \exp\left(\frac{-\|\bv_i-\bv_j\|^{2}_2}{2\sigma^2}\right),
    \end{align*}
    where $\bv_i$ and $\bv_j$ are the feature vectors of data points $i$ and $j$, and $\sigma$ is the width of Gaussian function. This setting meets our earlier assumption that the similarity matrix $\bW$ is positive semi-definite. 
    
    But what if some indefinite kernel has been adopted and hence $\bW$ is not necessarily positive semi-definite anymore? The following proposition provides a solution.

\begin{proposition}
\label{prop indefinite W}
    Suppose that the similarity matrix $\bW$ is indefinite. We can choose a sufficiently large $\alpha>0$ so that the new matrix
    \begin{equation}
    \label{eq:alpha}
        \widetilde{\bW} = \bW + \alpha \bD
    \end{equation}
    is positive semi-definite. Notice that such $\alpha$ must exist since 
    \begin{align}
    \label{alpha}
    \alpha = -\frac{\lambda_{\min}(\bW)}{\min_{i} d_{i}}
    \end{align}
    is a feasible choice. Then we can equivalently consider problem \eqref{prob:NCut} with $\widetilde\bW$ used in place of $\bW$, which can be readily addressed by the FPC algorithm.
\end{proposition}
\begin{IEEEproof}
    First of all, because $\bW$ is indefinite, its minimum eigenvalue $\lambda_{\min}(\bW)$ must be negative. With $\alpha$ in \eqref{alpha}, the matrix $\widetilde\bW$ can be rewritten as
    \begin{align}
        \widetilde\bW &= \bW -\frac{\lambda_{\min}(\bW)}{\min_{i} d_{i}}\bD\notag\\
        &= \bW - \lambda_{\min}(\bW)\bI+\bE 
    \end{align}
    where
    \begin{align}
        \bE=-\lambda_{\min}(\bW) \mathrm{diag}\left[\frac{d_1}{\min_i d_i}-1,\ldots,\frac{d_N}{\min_i d_i}-1\right].
    \end{align}
    It is evident that $\bW - \lambda_{\min}(\bW)\bI$ and $\bE$ are both positive semi-definite, so $\widetilde\bW$ is positive semi-definite too. Moreover, it is easy to see that problem \eqref{prob:NCut} is equivalent to
\begin{subequations}
\label{prob:NCut new alpha}
\begin{align}
    \underset{\bx_{k}}{\text{maximize}}&\quad K\alpha+\sum_{k=1}^K\frac{\bx_k^\top\bW\bx_k}{\bx_{k}^\top\bD\bx_{k}}
\label{prob:NCut new alpha:obj}\\
 \text{subject to}&\quad \eqref{cons:unique cluster},\;\eqref{cons:discrete value},
\end{align}
\end{subequations}
which can be further rewritten as 
\begin{subequations}
\label{prob:NCut new alpha re}
\begin{align}
    \underset{\bx_{k}}{\text{maximize}}&\quad \sum_{k=1}^K\frac{\bx_k^\top\widetilde\bW\bx_k}{\bx_{k}^\top\bD\bx_{k}}
\label{prob:NCut new alpha re:obj}\\
 \text{subject to}&\quad \eqref{cons:unique cluster},\;\eqref{cons:discrete value},
\end{align}
\end{subequations}
The proof is then completed. 
\end{IEEEproof}

However, when it comes to a large dataset for which the fully connected graph is impractical, we must adopt a sparse graph representation in the sense that the connection weight between most pairs of vertices equals zero. In this case, the above fully-connected kernel-based similarity matrix generation is no longer feasible. Instead, we propose using the \emph{anchor graph} \cite{liu2010large} to create a sparsely connected graph representation for large datasets, as described in the following. We first generate a total of $M$ anchors that captures the overall topology of the graph (according to some K-means procedure as specified in \cite{liu2010large}). For a given integer $1\le k\le M$, we obtain a bipartite graph matrix $\bF=[f_{ij}]\in\mathbb R^{N\times M}$ that quantifies the similarity between each vertex and each anchor as 
\begin{align}
\label{large similarity}
        f_{ij}=\left\{\begin{array}{ll}
    \!\!\frac{\|\bv_i-\bu^{(i)}_{k+1}\|^{2}_2-\|\bv_i-\bu^{(i)}_{j}\|^{2}_2}{k\|\bv_i-\bu^{(i)}_{k+1}\|^{2}_2-\sum_{r=1}^k\|\bv_i-\bu^{(i)}_{r}\|^{2}_2}, & \text{if}\;j\leq k;
    \vspace{0.5em}\\
    \!\!0 & \text{otherwise},
    \end{array}\right.
    \end{align}
    where $\bv_i$ is the feature vector of data points $i$, and $\bu^{(i)}_j$ is the feature vector of the $j$th closest anchor to data point $i$. The similarity matrix $\bW$ is then generated as
    \begin{align*}
        \bW=\bF\bm{R}^{-1}\bF^\top,
    \end{align*}
    where $\bm{R}$ is a diagonal matrix whose $j$th diagonal entry equals $\sum_{i=1}^{N}f_{ij}$. It can be shown that the resulting $\bW$ is PSD.}

\subsection{Computational Complexity}
{Finally, we examine the computational complexity of the FPC algorithm under the two graph constructions. For the fully connected graph, the update of $\{\bx_{k}\}$ as in \eqref{eq:update X} incurs a computational complexity of $\mathcal{O}(KN)$, while the update of $\{\bmu_k\}$ as in \eqref{eq:update H} incurs a computational complexity of $\mathcal{O}(KN^2)$.} If $\bW$ is generated by the indefinite kernel (for the small dataset case), then we need to compute the smallest eigenvalue of $\bW$ as in \eqref{alpha}. Rather than computing all the eigenvalues and then picking the smallest, which incurs $\mathcal{O}(N^3)$, we propose a more efficient way of computing $\lambda_{\min}$:
\begin{enumerate}
    \item Find the largest eigenvalue of $\|\bW\|_\mathrm{F}\bI-\bW$, denoted as $\lambda_1$, by the power method \cite{stewart1973}, where $\|\cdot\|_\mathrm{F}$ is the Frobenius norm.
    \item Compute the smallest eigenvalue as $\lambda_{\min}(\bW)=\|\bW\|_\mathrm{F}-\lambda_1$.
\end{enumerate}
Thus, the overall complexity of finding $\alpha$ as in \eqref{alpha} is $\mathcal O (N^2)$. To sum up, the per-iteration complexity of FPC equals $\mathcal{O}(KN^2)$, while the traditional SC algorithm incurs a complexity of $O(N^3)$.

{For the anchor graph $\bW=\bF\bm{R}^{-1}\bF^\top$, let $M$ be the number of anchors and let each sample be connected to $k$ anchors. The product $\bW\bX$ in \eqref{eq:update H} can be computed as $\bF\bm{R}^{-1}(\bF^\top\bX)$ without explicitly forming $\bW$, which incurs a complexity of $\mathcal{O}(NkK)$. Together with the $\mathcal{O}(K^2N)$ update in \eqref{eq:update X}, the per-iteration complexity of FPC is $\mathcal{O}(NkK+NK+MK)$. In comparison, SC obtains the spectral embedding from an $M\times M$ matrix with complexity $\mathcal{O}(Nk^2+M^3+NkK)$, followed by K-means clustering with per-iteration complexity $\mathcal{O}(NK^2)$. Thus, the overall complexity of SC is $\mathcal{O}(Nk^2+M^3+NkK+NK^2)$. To sum up, both algorithms scale linearly with $N$.}

\section{Experiments}
\label{sec:Experiments}

\begin{table}[t]
  \centering
  \caption{Datasets used for dividing $N$ data points into $K$ clusters.}
  \begin{threeparttable}
    \resizebox{\columnwidth}{!}{%
    \begin{tabular}{ccccc}
    \toprule
    Dataset & $N$ & $K$ & \# of features & Source\\
    \midrule
    Shuttle & 15 & 2 & 6 & UCI\cite{UCI}\\
    Balloons & 16 & 2 & 4 & UCI\cite{UCI}  \\
    Breast & 106 & 6 & 9 & UCI\cite{UCI}  \\
    Thyroid & 215 & 3 & 5 & UCI\cite{UCI}\\
    Office & 2533 & 10 & 800 & Github\cite{transferlearning.xyz} \\
    Splice & 3175 & 3 & 240 & UCI\cite{UCI} \\
    Rice & 3810 & 2 & 7 & UCI\cite{UCI} \\
    Landsat & 6435 & 7 & 36 & UCI\cite{UCI} \\
    {Epileptic} & {11500} & {5} & {178} & {UCI\cite{UCI}}\\
    STL10 & 13000 & 10 & 27648 & \cite{coates2011analysis}\\
    {Sensorless}\footnote{Sensorless denotes the
Sensorless Drive Diagnosis dataset.} & {58509} & {11} & {48} & {UCI\cite{UCI}}\\
    CIFAR10 & 60000 & 10 & 3072 & \cite{krizhevsky2009learning}\\
    {MNIST} & {70000} & {10} & {784} & \cite{lecun1998gradient}\\
    {RT-IoT2022} & {123117} & {12} & {83} & {UCI\cite{UCI}}\\
    \bottomrule
    \end{tabular}%
    }
    \end{threeparttable}
  \label{tab:dataset}%
\end{table}%

{We evaluate FPC on twelve benchmark datasets and use two mini-datasets only for the optimality tests, as summarized in Table~\ref{tab:dataset}. The compared methods fall into two categories:
\begin{itemize}
\setlength{\itemsep}{0pt}
\setlength{\parskip}{0pt}
\setlength{\parsep}{0pt}
\item \emph{Two-stage methods} solve a relaxed spectral problem before discretization: SC uses K-means \cite{NCut}, ISR uses improved spectral rotation \cite{chen2017scalable}, and DSSRC uses graph-cut-guided nonheuristic discretization \cite{zhang2023discretize}.
\item \emph{Direct NCut optimization methods} include FINC, which uses a scalar quadratic transform and heuristic subproblem updates \cite{FINC}; FCD, which applies block coordinate descent to the discrete assignments \cite{FCD}; and FPC, which combines the multidimensional quadratic transform with weighted bipartite matching.
\end{itemize}} 
{For small datasets, the similarity matrix is generated by using the Gaussian kernel. 
The width of the Gaussian function, denoted as $\sigma$, is set to $\frac{\sqrt{2}}{2}$. For large datasets, we use the anchor graph with the number of anchors $M=1024$. The anchor points are randomly selected. The parameter $k$ in \eqref{large similarity} is set to $20$. For the anchor-graph experiments, all compared methods operate on the factorized similarity representation without explicitly forming the full matrix $\bW$. Each method is evaluated over 20 runs. FINC, FCD, and FPC use the corresponding SC result for initialization. All the experiments have been carried out on a desktop equipped with 2.10 GHz CPU$\times12$. In each table, we highlight the best performance by the bold font.}

\subsection{Image Segmentation by Clustering}
We further try out the different clustering algorithms in the image segmentation task. The image dataset from \cite{iccv09data} is used. We first partition each image into 500 superpixels and then perform the color histogram and the local binary pattern analysis to extract features as in \cite{cheng2011multi}. As shown in Fig.~\ref{fig:seg} (next page), the clustering by FPC gives clearer boundaries of the objects than other methods.

\begin{figure*}[!t]
  \centering
  \includegraphics[width=0.96\linewidth]{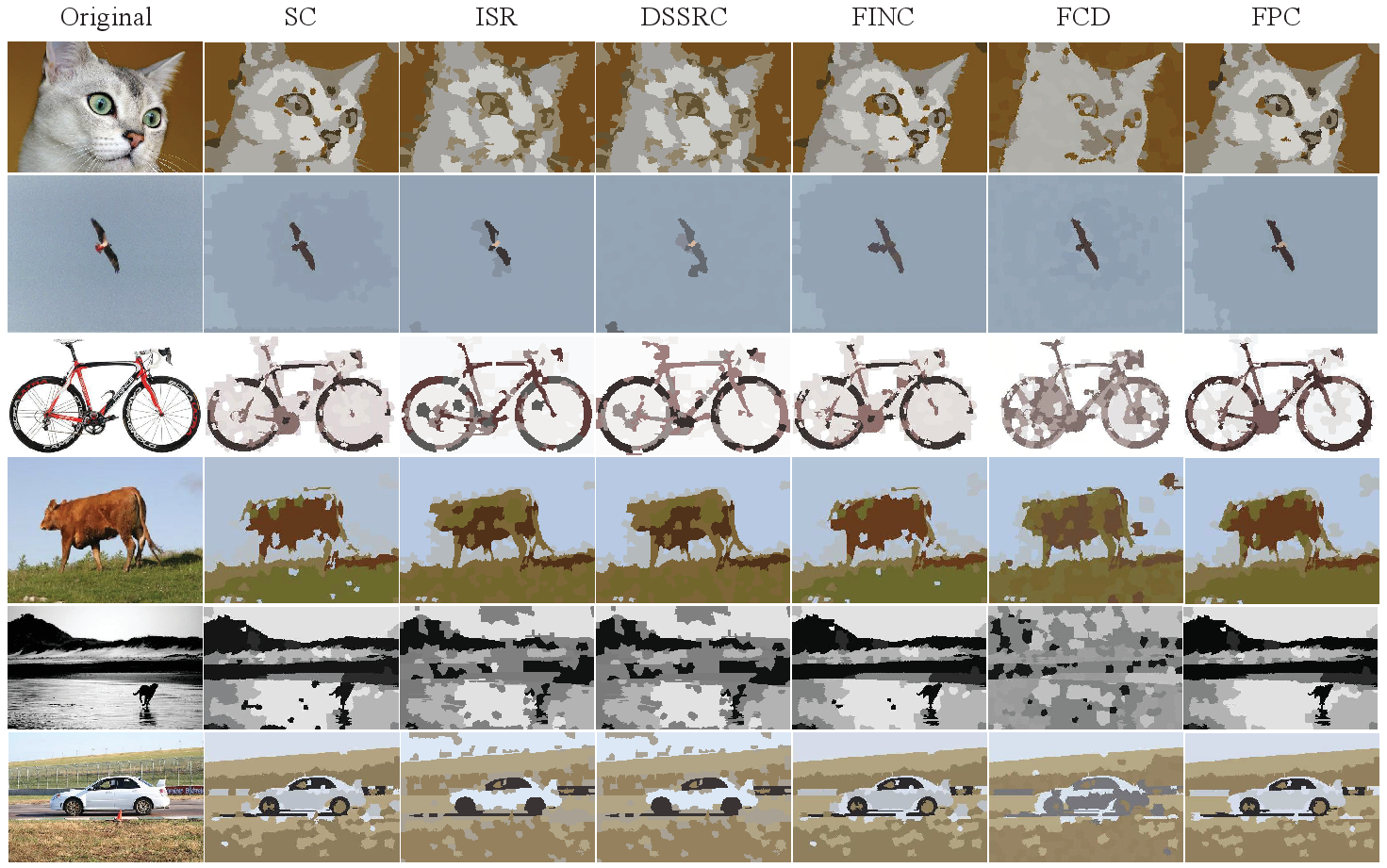}
  \caption{Image segmentation by the different algorithms.}
  \label{fig:seg}
\end{figure*}

\subsection{Optimality Tests}

\begin{figure*}[!t]
  \centering
  \begin{minipage}[t]{0.66\linewidth}
    \centering
    \subfigure[Dataset ``Shuttle'']{
      \includegraphics[width=0.474\linewidth]{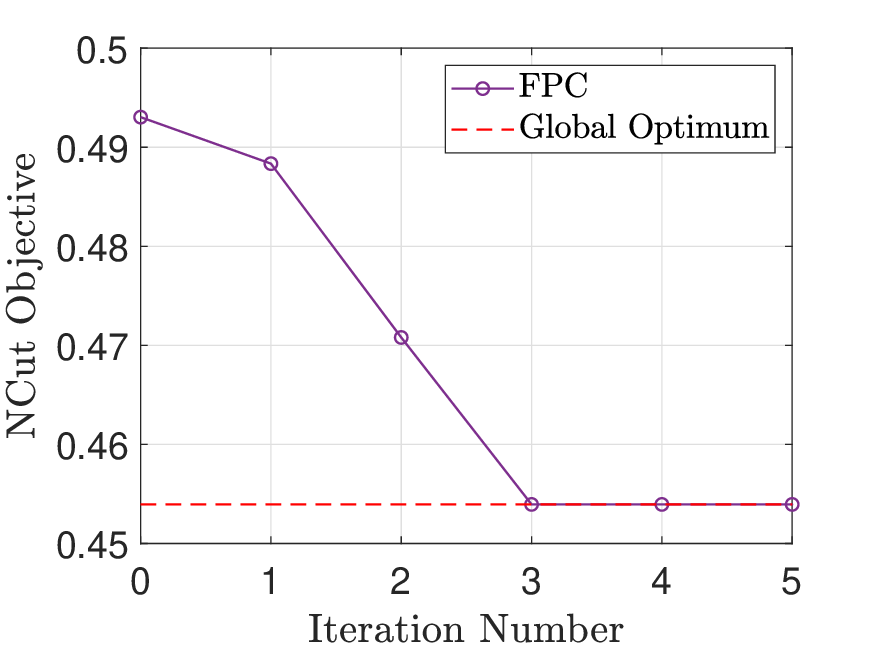}
    }
    \subfigure[Dataset ``Balloons'']{
      \includegraphics[width=0.474\linewidth]{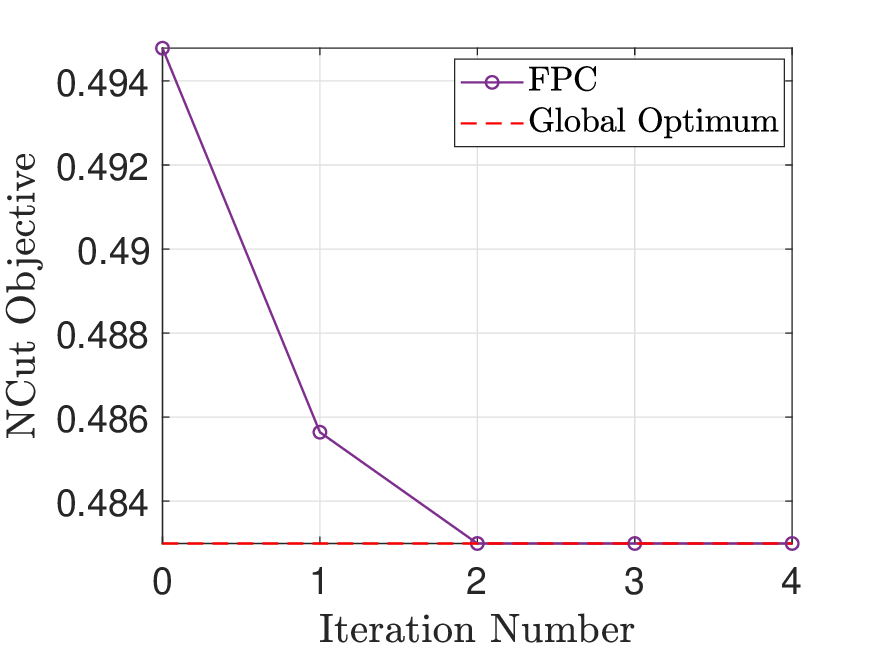}
    }
    \caption{Convergence in terms of the NCut objective value for two UCI datasets. FPC converges to the global optimum after only 3 iterations.}
    \label{fig:convergence:global}
  \end{minipage}
  \hfill
  \begin{minipage}[t]{0.315\linewidth}
    \centering
    \includegraphics[width=\linewidth]{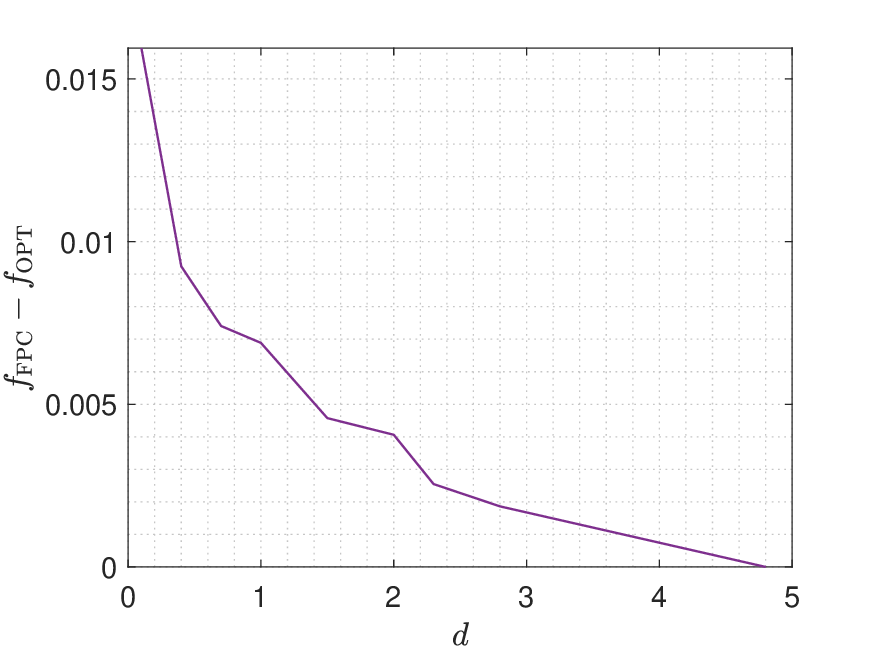}
    \caption{The gap between the NCut objective $f_\text{FPC}$ {achieved} by FPC and the global optimum $f_\text{OPT}$ for two Gaussian clusters $\mathcal N(0,\bm I_2)$ and $\mathcal N(d,\bm I_2)$.}
    \label{fig:converge_Gaussian}
  \end{minipage}
\end{figure*}

This subsection aims to evaluate the gap between the performance achieved by FPC and the global optimum. We use the exhaustive search to find the global optimum. As such, we consider two mini-datasets ``Shuttle'' and ``Balloons'' as shown in Table \ref{tab:dataset}.  Observe from Fig.~\ref{fig:convergence:global}  that FPC attains convergence to the global optimum after merely $3$ iterates.

Furthermore, we try out the FPC algorithm on the different clustering scenarios with the distance between two cluster centers being gradually reduced, and consequently the clustering problem becomes increasingly difficult. In this instance, two clusters are randomly generated as Gaussian mixtures. Assume that each data point is a 2D vector. The data points of the two clusters are randomly and independently generated according to the distributions $\mathcal N(0,\bm I_2)$ and $\mathcal N(d,\bm I_2)$, respectively, where $\bm I_2$ is the 2D identity matrix, and $d\ge0$ amounts to the distance between the two cluster centers. Observe from Fig.~\ref{fig:converge_Gaussian} that the solution obtained by our algorithm is fairly close to the global optimum.

\subsection{Sensitivity to Initialization}
\label{subsec:init_sensitivity}
\begin{figure}[!t]
  \centering
  \includegraphics[width=1\linewidth]{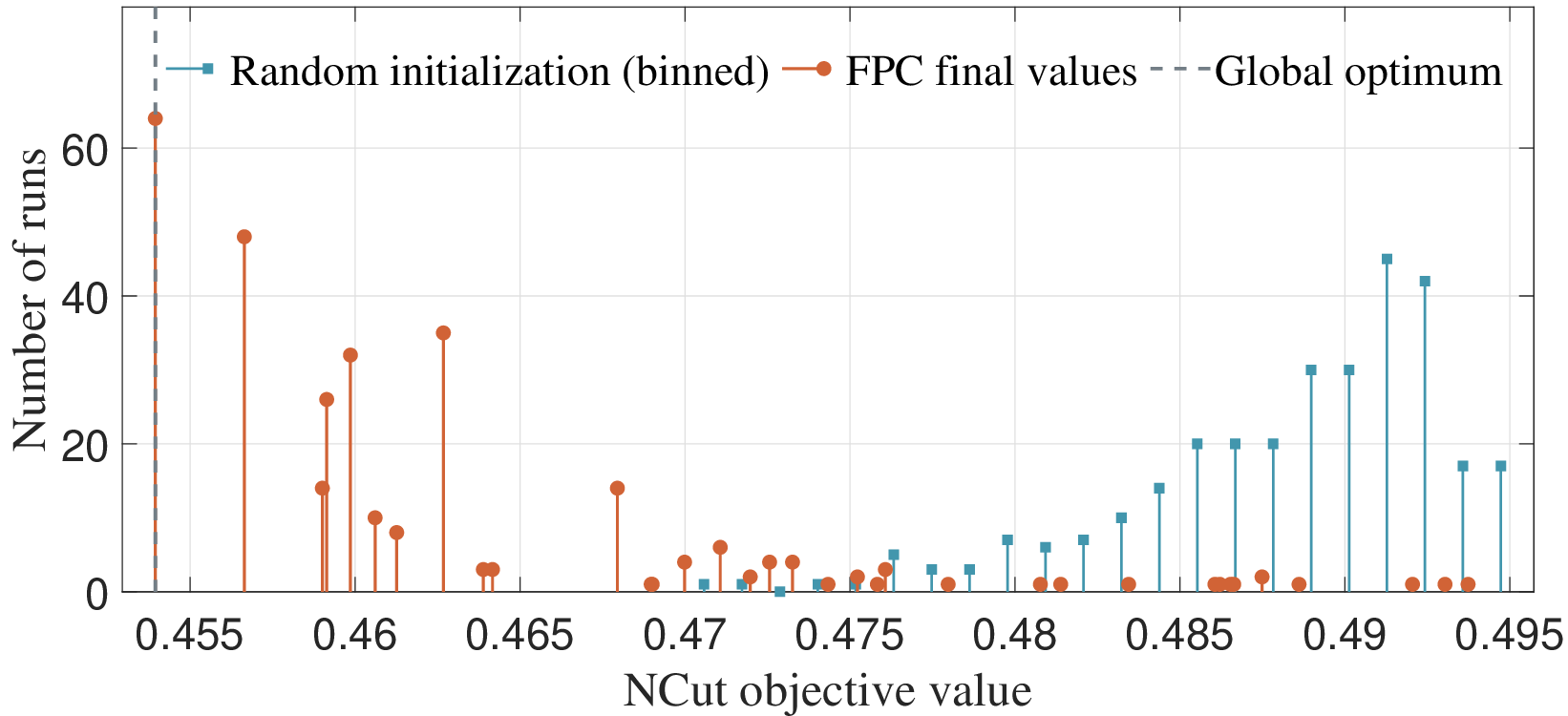}
  \caption{Sensitivity of FPC to random initialization on ''Shuttle'' dataset.}
  \label{fig:init_sensitivity}
\end{figure}
We further examine the sensitivity of FPC to initialization.
We consider the Shuttle dataset, for which the global optimum
can be obtained by exhaustive search, and run FPC from
multiple random feasible initializations. As shown in Fig.~\ref{fig:init_sensitivity} (next page),
the final NCut objective values are generally shifted toward lower values than the initial ones, and a considerable number of runs converge to the global optimum. The results also show that different initializations may lead to different local solutions, which motivates the use of multiple runs when evaluating the proposed algorithm.

\subsection{NCut Objective Optimization Performance}
{
\begin{table*}[!t]
\centering
\caption{NCut objective values achieved by the different algorithms on datasets constructed using fully connected and anchor graphs. A lower value indicates better performance.}
\label{tab:unified_ncut_objective}
\setlength{\tabcolsep}{2.5pt}
\renewcommand{\arraystretch}{1.05}
\resizebox{\textwidth}{!}{%
\begin{tabular}{lcccccccccccc}
\toprule
 & \multicolumn{6}{c}{Fully connected graph} & \multicolumn{6}{c}{Anchor graph} \\
\cmidrule(lr){2-7}\cmidrule(lr){8-13}
Method & Breast & Thyroid & Office & Splice & Rice & Landsat & Epileptic & STL10 & Sensorless & CIFAR10 & MNIST & RT-IoT2022 \\
\midrule
SC & \textbf{2.431325} & 0.983145 & 4.483929 & 0.997651 & 0.499193 & 2.490237 & 0.557434 & 1.844406 & 1.214324 & 2.073733 & 1.033129 & 0.254008 \\
ISR & \textbf{2.431325} & \textcolor[rgb]{0,.439,.753}{0.983144} & \textcolor[rgb]{1,0,0}{4.484089} & 0.997651 & 0.499193 & \textcolor[rgb]{1,0,0}{2.490253} & \textcolor[rgb]{1,0,0}{0.557458} & \textcolor[rgb]{1,0,0}{1.844957} & \textcolor[rgb]{1,0,0}{1.215714} & \textcolor[rgb]{1,0,0}{2.078630} & \textcolor[rgb]{1,0,0}{1.033709} & \textcolor[rgb]{1,0,0}{0.261397} \\
DSSRC & \textbf{2.431325} & \textcolor[rgb]{0,.439,.753}{0.983144} & \textcolor[rgb]{1,0,0}{4.484115} & \textcolor[rgb]{1,0,0}{0.997652} & 0.499193 & \textcolor[rgb]{1,0,0}{2.490246} & \textcolor[rgb]{1,0,0}{0.677723} & \textcolor[rgb]{1,0,0}{1.992010} & \textcolor[rgb]{1,0,0}{1.215714} & \textcolor[rgb]{1,0,0}{2.149170} & \textcolor[rgb]{1,0,0}{1.057807} & -- \\
FINC & \textcolor[rgb]{1,0,0}{2.432491} & \textcolor[rgb]{1,0,0}{0.988532} & \textcolor[rgb]{0,.439,.753}{4.483329} & 0.997651 & 0.499193 & 2.490237 & \textcolor[rgb]{1,0,0}{0.574284} & \textcolor[rgb]{1,0,0}{1.851028} & \textcolor[rgb]{0,.439,.753}{1.162947} & \textcolor[rgb]{0,.439,.753}{2.062307} & \textcolor[rgb]{0,.439,.753}{1.008371} & \textcolor[rgb]{0,.439,.753}{0.242925} \\
FCD & \textbf{2.431325} & 0.983145 & \textcolor[rgb]{0,.439,.753}{4.483670} & 0.997651 & 0.499193 & \textcolor[rgb]{0,.439,.753}{2.490138} & \textcolor[rgb]{0,.439,.753}{\textbf{0.555138}} & \textcolor[rgb]{0,.439,.753}{1.841004} & \textcolor[rgb]{0,.439,.753}{1.151903} & \textcolor[rgb]{0,.439,.753}{2.047859} & \textcolor[rgb]{0,.439,.753}{\textbf{0.995819}} & \textcolor[rgb]{0,.439,.753}{\textbf{0.241988}} \\
FPC & \textbf{2.431325} & \textcolor[rgb]{0,.439,.753}{\textbf{0.983118}} & \textcolor[rgb]{0,.439,.753}{\textbf{4.483264}} & \textcolor[rgb]{0,.439,.753}{\textbf{0.997638}} & \textcolor[rgb]{0,.439,.753}{\textbf{0.499192}} & \textcolor[rgb]{0,.439,.753}{\textbf{2.490049}} & \textcolor[rgb]{0,.439,.753}{0.555142} & \textcolor[rgb]{0,.439,.753}{\textbf{1.840981}} & \textcolor[rgb]{0,.439,.753}{\textbf{1.151889}} & \textcolor[rgb]{0,.439,.753}{\textbf{2.047775}} & \textcolor[rgb]{0,.439,.753}{\textbf{0.995819}} & \textcolor[rgb]{0,.439,.753}{\textbf{0.241988}} \\
\bottomrule
\end{tabular}%
}
\begin{minipage}{\textwidth}
\footnotesize Blue and red respectively denote improvement and degradation relative to SC. Bold values indicate the best displayed result on each dataset; {-- indicates that the method did not finish within the practical time budget.}
\end{minipage}
\end{table*}
}

{For each method, we perform 20 runs and report the solution attaining the smallest NCut objective value, using the same 20-run budget and selection rule for all compared methods. We then evaluate the performance of the different algorithms in minimizing $\mathrm{ncut}(\mathcal V_1,\mathcal V_2,\ldots,\mathcal V_K)$ as defined in \eqref{ncut def}  under the two graph-construction settings. Table~\ref{tab:unified_ncut_objective}  summarizes the results. 
The fully connected graph is used for the six smaller datasets, whereas the anchor graph is used for the six medium- and large-scale datasets. As expected, the direct NCut optimization methods generally attain lower objective values than the two-stage methods. This is because the latter first solve a relaxed spectral problem and then recover discrete assignments through a post-processing step, which does not directly optimize the original discrete NCut objective.  Among the direct optimization methods, FPC achieves the best or tied-best result on eleven of the twelve datasets. {For instance, compared with SC, FPC reduces the NCut objective by 5.14\%, 4.73\%, and 3.61\% on Sensorless, RT-IoT2022, and MNIST, respectively. Compared with DSSRC, the corresponding reductions reach 18.09\% on Epileptic, 7.58\% on STL10, and 5.86\% on MNIST.} FINC can even degrade the SC initialization on several datasets, since its transformed subproblem is handled heuristically and is not guaranteed to be solved optimally at each iteration, as discussed in Section~\ref{subsec:conventional FP methods}.}

\begin{figure*}[th!]
\centering
\subfigure[dataset ``Breast'']{
\includegraphics[width=0.32\linewidth]{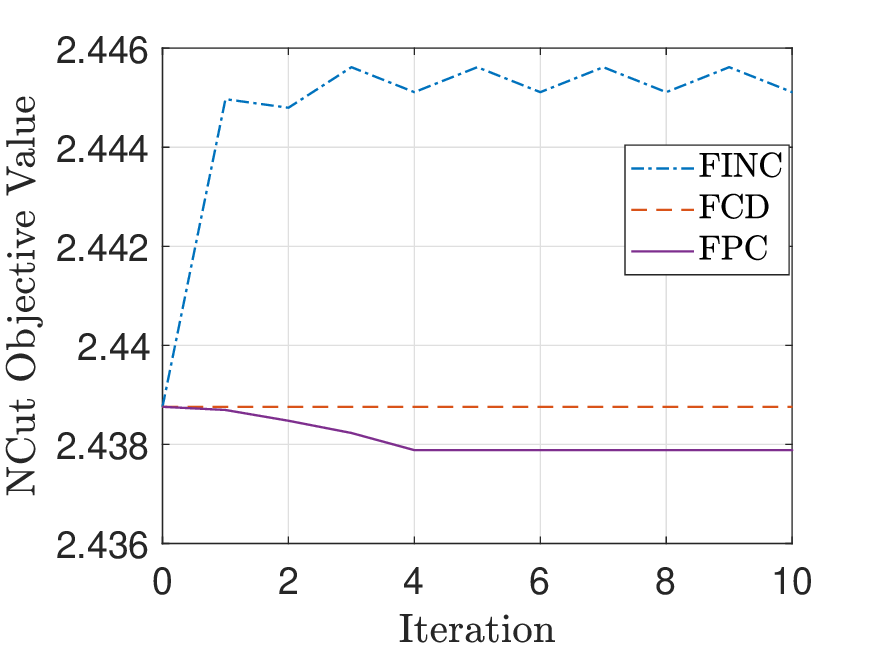}}
\subfigure[dataset ``Splice'']{
\includegraphics[width=0.32\linewidth]{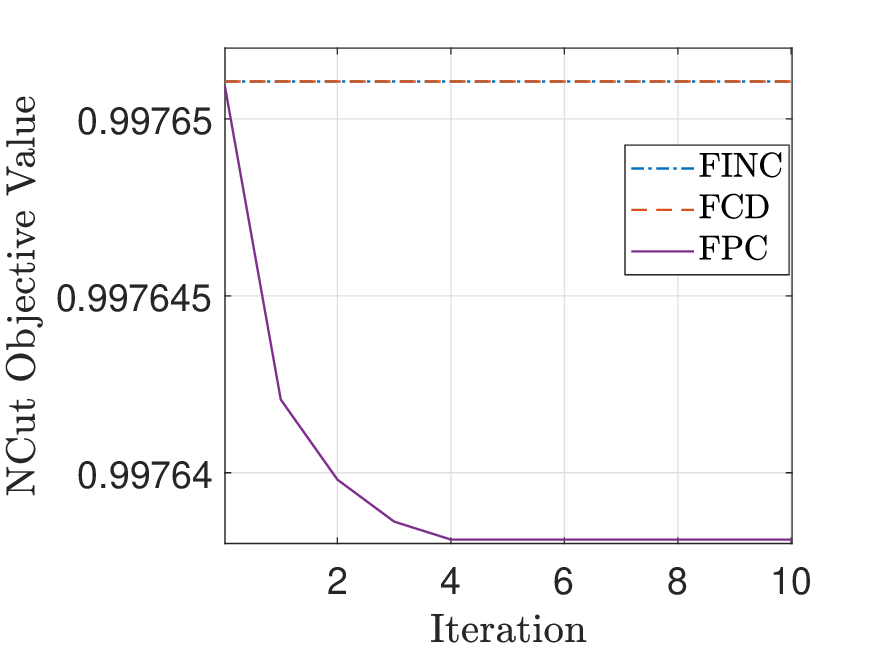}}
\subfigure[dataset ``Landsat'']{
\includegraphics[width=0.32\linewidth]
{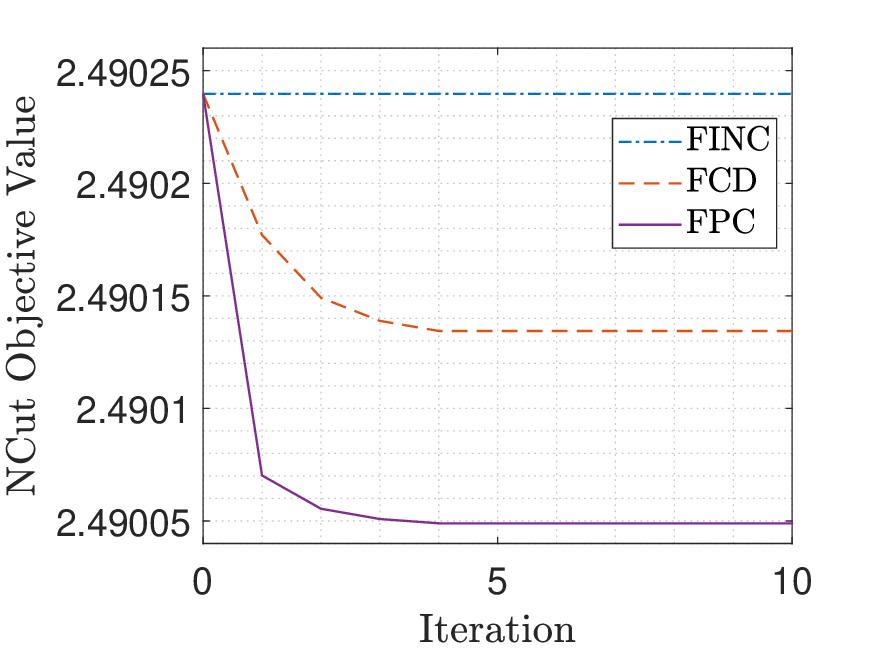}}
\caption{Convergence in terms of the NCut objective value for the different datasets with the similarity matrix generated as fully connected graph.}
\label{fig:convergence_fully_connected_graph}
\end{figure*}

\begin{figure*}[t]
\centering
\subfigure[dataset ``Epileptic'']{
\includegraphics[width=0.32\linewidth]{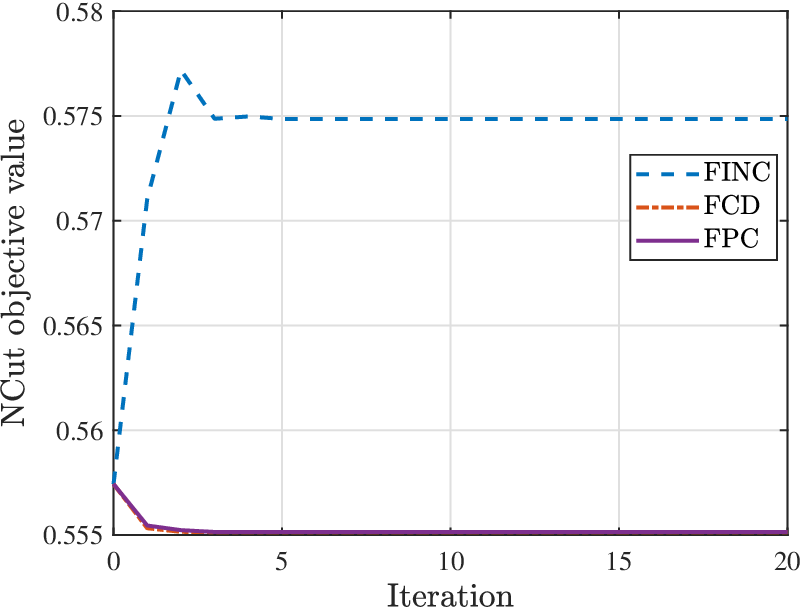}}
\subfigure[dataset ``STL10'']{
\includegraphics[width=0.32\linewidth]{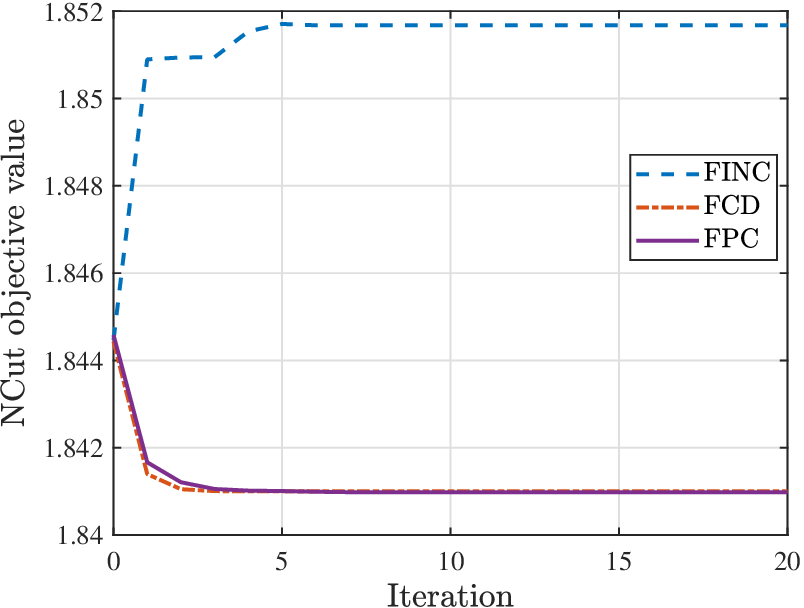}}
\subfigure[dataset ``CIFAR10'']{
\includegraphics[width=0.32\linewidth]{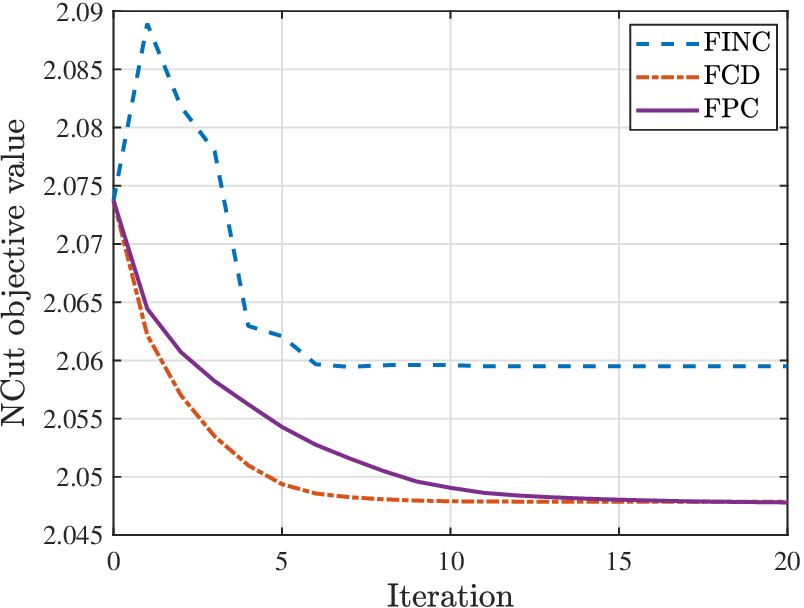}}
\caption{Convergence in terms of the NCut objective value for the different datasets with the similarity matrix generated as anchor graph.}
\label{fig:convergence_anchor_graph}
\end{figure*}

{
\begin{table*}[th!]
\centering
\caption{Clustering performance achieved by the different algorithms on datasets constructed using fully connected and anchor graphs. A higher value indicates better performance.}
\label{tab:unified_clustering_metrics}
\setlength{\tabcolsep}{2.5pt}
\renewcommand{\arraystretch}{1.05}
\resizebox{\textwidth}{!}{%
\begin{tabular}{llcccccccccccc}
\toprule
 &  & \multicolumn{6}{c}{Fully connected graph} & \multicolumn{6}{c}{Anchor graph} \\
\cmidrule(lr){3-8}\cmidrule(lr){9-14}
Metric & Method & Breast & Thyroid & Office & Splice & Rice & Landsat & Epileptic & STL10 & Sensorless & CIFAR10 & MNIST & RT-IoT2022 \\
\midrule
\multirow{6}{*}{ACC} & SC & \textbf{0.509434} & 0.888372 & 0.350178 & 0.664882 & 0.896325 & 0.668500 & 0.378261 & 0.205769 & 0.259054 & 0.196583 & 0.464271 & 0.229221 \\
  & ISR & \textbf{0.509434} & 0.883721 & 0.346625 & 0.664567 & 0.896325 & 0.668500 & 0.378087 & 0.205615 & 0.259071 & 0.200283 & 0.478629 & 0.228336 \\
  & DSSRC & \textbf{0.509434} & 0.883721 & 0.347809 & 0.660787 & 0.896588 & 0.677000 & 0.381478 & \textbf{0.217769} & 0.259071 & 0.206983 & 0.462614 & -- \\
  & FINC & \textbf{0.509434} & 0.604651 & \textbf{0.350572} & 0.664882 & 0.896325 & 0.668500 & \textbf{0.382870} & 0.206000 & 0.265771 & \textbf{0.209117} & 0.479329 & 0.235248 \\
  & FCD & \textbf{0.509434} & 0.888372 & 0.348993 & 0.664882 & 0.896325 & 0.654500 & 0.377826 & 0.206615 & \textbf{0.266198} & 0.198717 & 0.480314 & \textbf{0.235361} \\
  & FPC & \textbf{0.509434} & \textbf{0.893023} & 0.346230 & \textbf{0.673701} & \textbf{0.898163} & \textbf{0.683000} & 0.378000 & 0.206923 & 0.266147 & 0.198667 & \textbf{0.480329} & \textbf{0.235361} \\
\midrule
\multirow{6}{*}{NMI} & SC & \textbf{0.480282} & 0.511042 & 0.235126 & 0.285693 & 0.512112 & 0.611831 & 0.177041 & 0.108122 & \textbf{0.297855} & 0.067306 & 0.446930 & 0.268947 \\
  & ISR & \textbf{0.480282} & 0.495642 & 0.227571 & 0.285773 & 0.512112 & 0.598479 & 0.176937 & 0.108185 & 0.297419 & 0.075624 & 0.448668 & 0.268627 \\
  & DSSRC & \textbf{0.480282} & 0.495642 & 0.227507 & 0.283410 & 0.512938 & 0.589606 & \textbf{0.177504} & \textbf{0.117170} & 0.297419 & \textbf{0.085738} & 0.439852 & -- \\
  & FINC & 0.460500 & 0.270826 & \textbf{0.235585} & 0.285693 & 0.512112 & 0.611831 & 0.171389 & 0.107785 & 0.296373 & 0.073012 & 0.471641 & 0.269156 \\
  & FCD & \textbf{0.480282} & 0.511042 & 0.233231 & 0.285693 & 0.512112 & 0.612702 & 0.175288 & 0.108774 & 0.296980 & 0.070436 & 0.475403 & \textbf{0.269244} \\
  & FPC & \textbf{0.480282} & \textbf{0.526925} & 0.234607 & \textbf{0.295071} & \textbf{0.518077} & \textbf{0.618908} & 0.175697 & 0.108945 & 0.297035 & 0.070504 & \textbf{0.475513} & \textbf{0.269244} \\
\midrule
\multirow{6}{*}{ARI} & SC & \textbf{0.311287} & 0.623665 & \textbf{0.143893} & 0.268918 & 0.627953 & 0.478430 & 0.133450 & 0.055283 & 0.118808 & 0.035444 & 0.301169 & 0.080775 \\
  & ISR & \textbf{0.311287} & 0.608237 & 0.137564 & 0.268437 & 0.627953 & 0.463606 & 0.133174 & 0.055500 & 0.118674 & 0.038022 & 0.316367 & 0.080125 \\
  & DSSRC & \textbf{0.311287} & 0.608237 & 0.137418 & 0.264810 & 0.628785 & 0.460109 & 0.140195 & \textbf{0.059465} & 0.118674 & \textbf{0.044224} & 0.297796 & -- \\
  & FINC & 0.296489 & 0.206981 & 0.141879 & 0.268918 & 0.627953 & 0.478430 & \textbf{0.141443} & 0.054702 & 0.119099 & 0.039179 & 0.335177 & 0.081637 \\
  & FCD & \textbf{0.311287} & 0.623665 & 0.142037 & 0.268918 & 0.627953 & 0.485465 & 0.131139 & 0.055159 & 0.119302 & 0.039380 & 0.337417 & \textbf{0.081653} \\
  & FPC & \textbf{0.311287} & \textbf{0.639239} & 0.140155 & \textbf{0.280361} & \textbf{0.633763} & \textbf{0.522563} & 0.131358 & 0.055301 & \textbf{0.119322} & 0.039459 & \textbf{0.337461} & \textbf{0.081653} \\
\bottomrule
\end{tabular}%
}
\begin{minipage}{\textwidth}
\footnotesize Bold values indicate the best displayed result on each dataset; {-- indicates that the method did not finish within the practical time budget.}
\end{minipage}
\end{table*}
}

{We further compare the convergence behavior of the three direct optimization methods in Figs.~\ref{fig:convergence_fully_connected_graph} and~\ref{fig:convergence_anchor_graph} (next page). FPC monotonically improves upon the starting point on all six representative datasets under both graph constructions, while FCD remains at the SC initialization on Breast and Splice. This indicates that our method is more capable of escaping from local solutions than FCD. FINC exhibits less stable behavior: its objective can increase or oscillate, as observed on Breast and Epileptic, and it generally converges to a higher objective value than FPC. Observe also that FPC has high efficiency, with the convergence attained in fewer than 10 iterations for the datasets Breast and Splice with $N<4000$ and in fewer than 20 iterations for the dataset STL10 with $N=13000$.}

\subsection{Clustering Performance}

 Aside from the {NCut} optimization criterion, the following commonly used performance metrics in practice are considered for the different clustering algorithms: the accuracy (ACC), the normalized mutual information (NMI) \cite{NMI}, and the adjusted random index (ARI) \cite{ARI1}.  For consistency with the preceding NCut comparison, these metrics are evaluated on the same solutions selected according to the NCut objective over the 20 runs. Unlike $\mathrm{ncut}(\mathcal V_1,\mathcal V_2,\ldots,\mathcal V_K)$, the above metrics are proportional to the performance, i.e., the higher metric value, the better clustering. {The results are summarized in Table~\ref{tab:unified_clustering_metrics} on the next page. Although these performance metrics are not directly tied to the NCut objective, FPC still achieves the highest number of top rankings across all evaluation metrics. Specifically, it attains the best performance in ACC on 7 datasets, in NMI on 7 datasets, and in ARI on 8 datasets. }

\subsection{Saturation of NCut Objective Function for FPC}
\begin{figure*}[!t]
\centering
\subfigure[dataset ``Epileptic'']{
\includegraphics[width=0.32\linewidth]{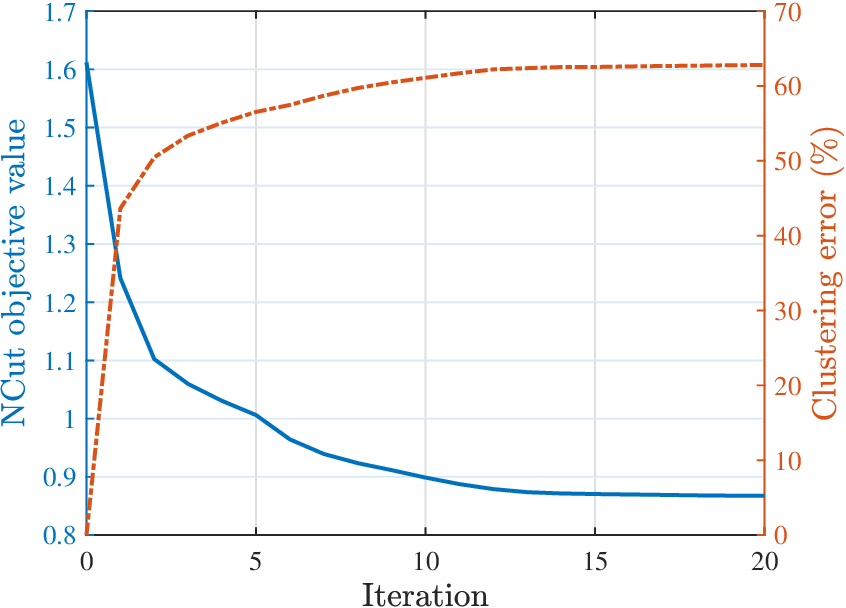}}
\subfigure[dataset ``STL10'']{
\includegraphics[width=0.32\linewidth]{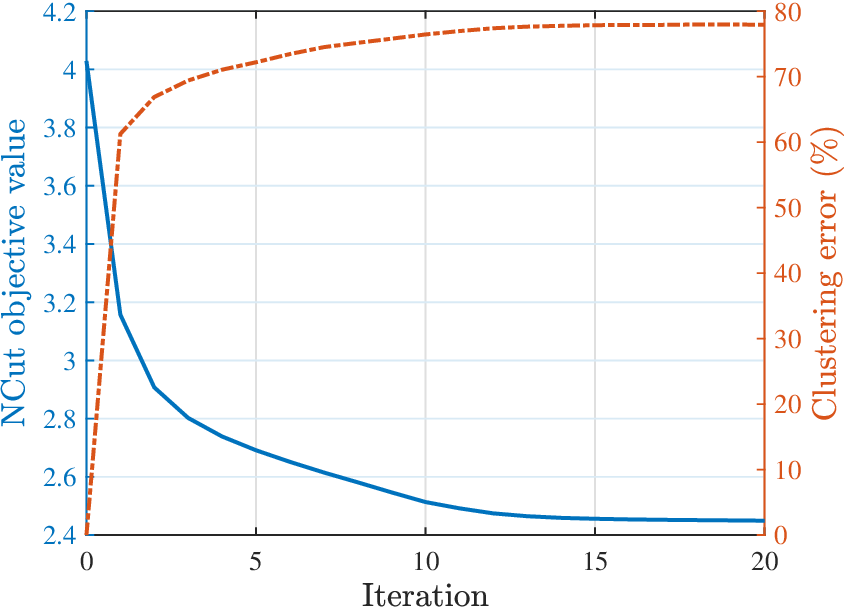}}
\subfigure[dataset ``CIFAR10'']{
\includegraphics[width=0.32\linewidth]{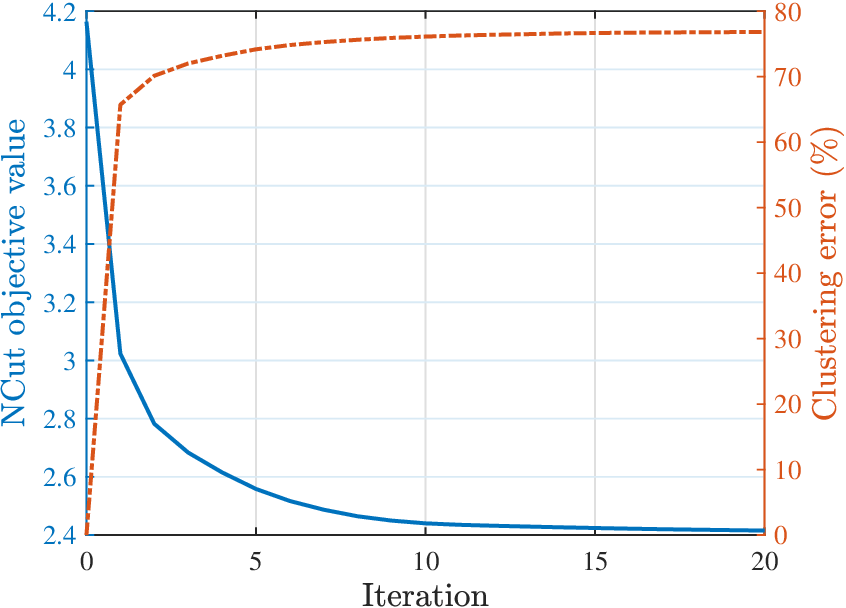}}
\caption{{Saturation behavior of FPC on Epileptic, STL10, and CIFAR10. FPC is initialized with the ground-truth labels, for which the clustering error is 0\%. The objective value decreases while the error increases.}}
\label{fig:saturation}
\end{figure*}

{
For many datasets, as the objective function is optimized, the accuracy (quality) of clustering is usually improved. But this works only up to a point. Beyond that, further optimization of the objective will not improve the quality of clustering because the objective function does not necessarily model the data in fine detail. This characteristic of clustering objective function is formalized as the \emph{saturation of objective function} according to \cite{ding2002cluster}. The saturation of objective function is also examined in our problem case. {For instance, FPC attains the lowest NCut objective values on Office, STL10, Sensorless, and CIFAR10. However, its ACC on these datasets
remains suboptimal.}

The saturation tests are conducted for the following datasets: Epileptic, STL10, and CIFAR10. The metric that reflects the accuracy quality is: 
\begin{align}
    \mathrm{Error} = 1-\mathrm{Accuracy}.
\end{align}
The convergence curves in terms of the NCut objective value and the error are shown in Fig.~\ref{fig:saturation}. The error starts at 0\% and then increases, while the objective function value improves (decreases). The final converged objective function value is termed the saturation value, and the converged clustering error (accuracy) is termed the saturation error (accuracy). Since NCut optimization prefers the converged solution even though its external error is higher, an error below the observed saturation error cannot be meaningfully interpreted as indicating a better solution under the NCut objective alone. In other words, NCut cannot reliably distinguish or rank solutions within this lower-error region according to their agreement with the ground truth. This criterion-specific interpretation does not imply that a lower clustering error lacks practical value; rather, it shows that such an advantage is not represented by the NCut objective.
}

\subsection{Time Consumption}
{
\begin{figure*}[t!]
  \centering
  \includegraphics[width=0.8\linewidth]{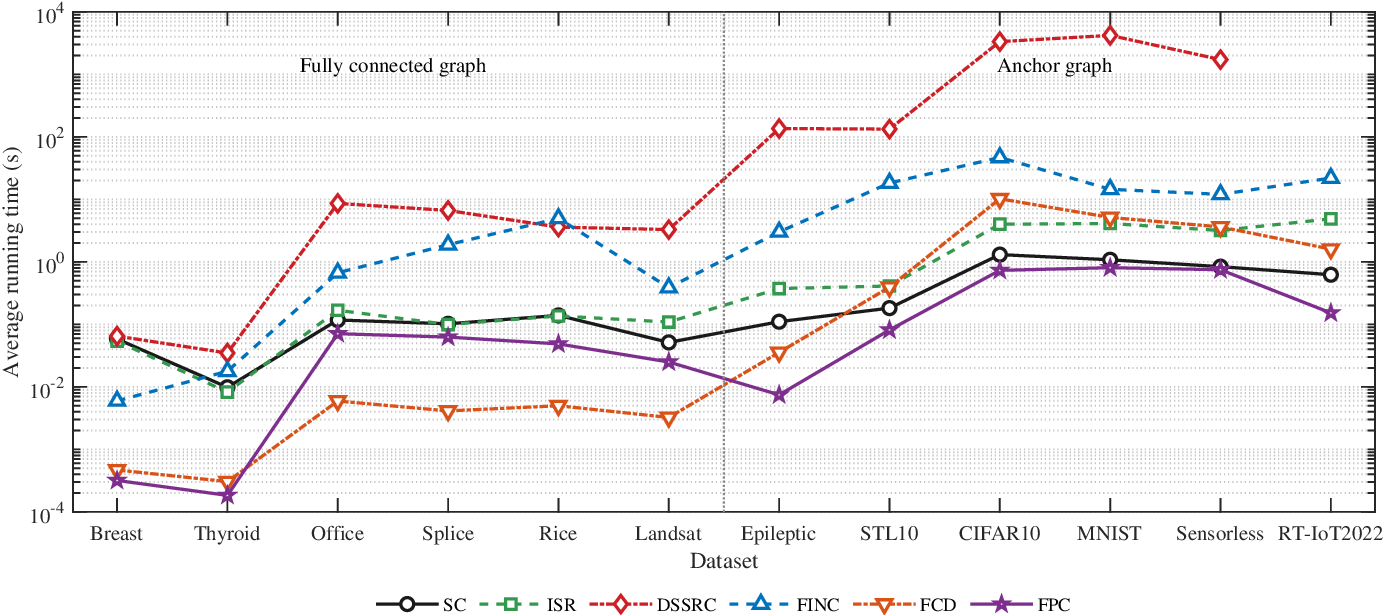}
  \caption{Average running time of the different algorithms on datasets constructed using fully connected and anchor graphs.}
  \label{fig:running time}
\end{figure*}
}

{Moreover, Fig.~\ref{fig:running time} compares the average running time under the two graph-construction settings. For a fair runtime comparison, the same timing convention is adopted for all methods, and the reported runtime includes only the corresponding optimization stage, excluding the initialization and common preprocessing stages. FPC is consistently faster than all three two-stage methods, i.e., SC, ISR, and DSSRC, on every dataset for which results are available.} {Quantitatively, under the fully connected graph setting, the average running time of FPC is 56.58\%, 63.87\%, 99.07\%, and 97.38\% lower than those of SC, ISR, DSSRC, and FINC, respectively. Under the anchor graph setting, FPC requires 38.92\%, 85.10\%, and 97.83\% less average running time than SC, ISR, and FINC, respectively. } {DSSRC is generally the most time-consuming baseline because its discrete mapping procedure is computationally expensive; on RT-IoT2022, it did not finish within a practical time budget, and hence no result is reported. FCD has the shortest running time on several datasets because it may terminate early at a suboptimal point. Overall, FPC achieves a favorable balance between computational efficiency and NCut objective improvement.}

\FloatBarrier

\section{Conclusion}
\label{sec:conclusion}

This work proposes a novel application of multidimensional FP to the NCut clustering, differing from the previous works that rely on the traditional scalar-ratio FP such as Dinkelbach's transform and quadratic transform. The main merit of using multidimensional FP is that the new 0-1 problem can be efficiently solved via linear search. On the analytical side, the resulting FPC algorithm can be interpreted as an MM procedure with provable monotonic convergence in terms of the NCut optimization criterion. Furthermore, we characterize the convergence rate of FPC. {The present work focuses on the multiclass NCut objective, while other graph partitioning criteria are beyond its main scope. Extending the proposed FP-based framework to online, streaming, multi-view, or distributed graph clustering would be an interesting direction for future research}.

\bibliographystyle{IEEEtran}
\bibliography{IEEEabrv,ref}

\label{sec:refs}

\end{document}